\documentclass{article}
\pdfoutput=1

\usepackage{arxiv}

\usepackage{microtype}
\usepackage{graphicx}
\usepackage{caption}
\usepackage{subcaption}
\usepackage[utf8]{inputenc} 
\usepackage[T1]{fontenc}    
\usepackage{hyperref}       
\usepackage{url}            
\usepackage{booktabs}       
\usepackage{amsfonts}       
\usepackage{amsmath}

\usepackage{nicefrac}       
\usepackage{microtype}      
\usepackage{enumitem}
\usepackage{multirow}
\usepackage[table,xcdraw]{xcolor}
\usepackage{algorithm}
\usepackage{algpseudocode}
\usepackage{sidecap}
\usepackage{amsthm}
\usepackage{MnSymbol}%
\usepackage{wasysym}%

\newtheorem{theorem}{Theorem}
\newtheorem{corollary}[theorem]{Corollary}
\newtheorem{definition}[theorem]{Definition}

\newcommand{\ssection}[1]{\vspace{-0.25cm} \section{#1} \vspace{-0.3cm}}
\newcommand{\ssubsection}[1]{\vspace{-0.25cm} \subsection{#1} \vspace{-0.25cm}}
\title{The trade-offs of model size in large recommendation models : A 10000 $\times$ compressed criteo-tb DLRM model (100 GB parameters to mere 10MB)}
\newcommand{\shorttitle}{The trade-offs of model size in large recommendation models}

\author{%
  Aditya Desai \thanks{
  Department of Computer Science,
  Rice University,
  Houston, Tx 77005 }\\
  \texttt{apd10@rice.edu} \\
  \And 
  Anshumali Shrivastava \footnotemark[1] \thanks{ThirdAI Corp. Houston, Texas} \\
  \texttt{as143@rice.edu} \\
}

\begin{document}

\maketitle
\begin{abstract}
Embedding tables dominate industrial-scale recommendation model sizes, using up to terabytes of memory. A popular and the largest publicly available machine learning MLPerf benchmark on recommendation data is a Deep Learning Recommendation Model (DLRM) trained on a terabyte of click-through data. It contains 100GB of embedding memory (25+Billion parameters). DLRMs, due to their sheer size and the associated volume of data, face difficulty in training, deploying for inference, and memory bottlenecks due to large embedding tables. This paper analyzes and extensively evaluates a generic parameter sharing setup (PSS) for compressing DLRM models. We show theoretical upper bounds on the learnable memory requirements for achieving $(1 \pm \epsilon)$ approximations to the embedding table. Our bounds indicate exponentially fewer parameters suffice for good accuracy. To this end, we demonstrate a PSS DLRM reaching 10000$\times$ compression on criteo-tb without losing quality. Such a compression, however, comes with a caveat. It requires 4.5 $\times$ more iterations to reach the same saturation quality. The paper argues that this tradeoff needs more investigations as it might be significantly favorable. Leveraging the small size of the compressed model, we show a 4.3$\times$ improvement in training latency leading to similar overall training times. Thus, in the tradeoff between system advantage of a small DLRM model vs. slower convergence, we show that scales are tipped towards having a smaller DLRM model, leading to faster inference, easier deployment, and similar training times.
\end{abstract}
\ssection{Introduction}
Recently, recommendation systems have emerged as one of the largest workloads in machine learning \cite{archimpl}. Recommendation systems form the backbone of a good user experience on online platforms such as e-commerce and web search, where there is a flood of information. Thus, considerable effort goes into building recommendation systems. Deep learning recommendation models give a state-of-the-art performance. However, recommendation models suffer from a critical challenge - sparse features with millions of categorical values\cite{DLRM19, mudigere2021high} These state-of-the-art \cite{DLRM19, DCN17, song2019autoint, guo2017deepfm, huang2019fibinet, lian2018xdeepfm} methods learn a dense representation of the categorical values in a parameter structure called {\em embedding table}.



Most parameters in recommendation models come from embedding tables. For example, in the popular criteo-tb MLPerf benchmark model, the embedding tables are around 100GB, whereas other parameters only amount to 10MB. Industrial-scale recommendation models are one of the largest models built. The size of the embedding table can go as large as hundreds of terabytes. For example, a research article from Facebook discusses the training of a model of size 50TB over 128 GPUs \cite{mudigere2021high}.
The scale of these models leads to some unfavorable effects -  slower inference time, slower training time per iteration, and significant engineering challenges in training/deployment. In light of these issues, many works have investigated learning of compressed representation of embedding tables using various principles : (1) compositional embeddings \cite{QuoRemTrick19} (2) exploiting power-law in the observed frequencies of tokens \cite{MDTrick19, md1,md2,md3,md4,md5} (3)  low-rank decomposition \cite{ttrec}, and weight sharing methods \cite{hashtrick, robez}. Parameter-sharing methods will be the focus of our paper.

The general idea in machine learning has primarily shifted to larger models and more data. Larger models lead to better capacity, faster convergence, and better generalization. However, the community is realizing that the current route to DL success is unsustainable\cite{thompson2021deep}. It stands to reason that we must proceed cautiously on this path of building larger models. The large-scale nature of DLRM comes from blown-up embedding parameters - a seemingly inadvertent effect of naive usage of embedding tables. This paper evaluates the tradeoffs between smaller models achieved by parameter sharing methods and large models using naive embedding tables. We find that compressed DLRM models have seemingly no downside in training efficiency and quality.

This paper provides a theoretical answer to the question, "How much can embedding tables be compressed without significantly losing quality." We present a general parameter sharing setup (PSS) for embedding tables and prove that we only require memory logarithmic $(\mathcal{O}(1/\epsilon^2 \max(d^2, d.\log(n))))$ in $n$ to approximate an embedding table $E \in R^{n \times d}$ up to a factor of $\epsilon$ for subsequent computations. This result motivates us to evaluate the extents of compression that can be achieved for embedding tables. We find that we can obtain 10000$\times$ compression (10MB embedding tables) on the criteo-tb dataset to achieve the same target quality - an order of magnitude higher compression than previously reported. These small models improve inference times and training time per iteration, lower training costs, eliminate engineering challenges and make the models easy to deploy on low-resource devices. While large compression is possible in theory and as we will show in our results, one important aspect of training compressed models needs to be solved - slow convergence.

\begin{figure}
        \centering
        \caption{This not-to-scale illustration shows tradeoffs of having larger parameters ( and hence the memory footprint of the model ). Each line has its own scale. As we move towards the right, memory increases, and we need to have memory farther away from computational hardware. This increases the latency and thus inference time and training time per iteration. The energy cost also increases due to larger active hardware. We observe that the convergence becomes faster as the number of parameters increases. The solid lines show the overall training time and costs, showing that more parameters are not always better.}   
    \includegraphics[scale=0.16]{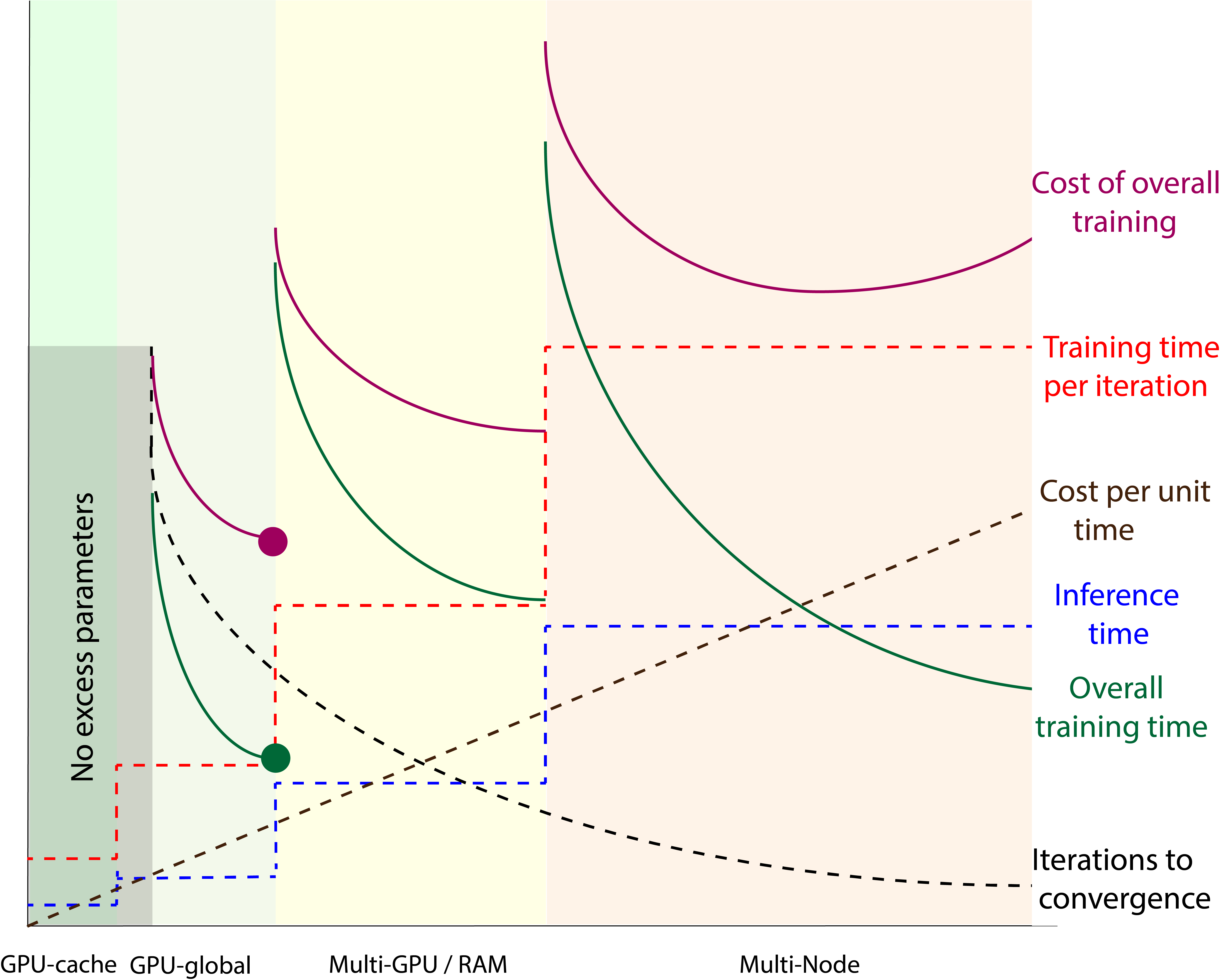}
    \vspace{-0.5cm}
     \label{fig:tradeoff}
\end{figure}

A general observation in machine learning is that the number of iterations required to converge to a target quality decreases with the increase in parameters. We also observe this effect in the training of parameter-shared recommendation models. The higher the compression, the slower the convergence:  a concern since training times are one of the significant aspects of the model. However, a highly compressed model completely changes the storage of model parameters (e.g., multiple nodes to a single node, RAM to GPU, etc.). Hence, we observe a steep decrease in the latency of embedding lookup/gradient update operations. An illustration of this tradeoff is presented in figure \ref{fig:tradeoff}. The question of how this improvement in latency compensates for increased training iterations is worth investigating. Unfortunately, existing parameter-sharing methods do not demonstrate much training time per iteration improvement due to various reasons such as high algorithmic complexity, poor cache efficiency, and sub-optimal implementation strategies. This paper talks about various aspects of implementation. With our PSS, we demonstrate that, indeed, the advantage of faster training time per iteration almost compensates for the slower convergence in compressed models. For example a 1,000$\times$ compressed model trains in $1.1 \times$ time whereas a 10,000$\times$ compressed model trains at $1.04 \times$ time as compared to original model.

This paper provides a theoretical peek into the compression of large recommendation models. Also, it resolves one of the major roadblocks in training compressed recommendation models by leveraging the system advantage of compressed models. We show 10000$\times$ compression ($10\times$ higher than best known) on DLRM on criteo-tb, making the total model size 20MB. Such a small DLRM can now be easily deployed on client devices that are resource constrained. 

\ssection{Theory of parameter-shared embedding tables}
\label{sec:theory}
\vspace{0.25cm}
\ssubsection{General parameter sharing embedding table setup}
\begin{definition}[\textbf{Parameter shared setup}]
Under parameter shared setup (PSS) for embedding table of size $n \times d$, we have a set of weights $M$ and recovery function $\mathcal{M} : R^{|M|} \times \{0,..,n-1\} \rightarrow R^d$. The embedding of a token $i \in \{0,..,n-1\}$ is recovered as $\mathcal{M}(M, i)$
\end{definition}
M can be represented with any memory layout(eg. 2D/1D array). We denote the size of $M$ as $|M|$ and memory to store $\mathcal{M}$ as $|\mathcal{M}|$. Note the inherent trade-off between $|M|$ and $|\mathcal{M}|$. We can make $M$ large in order to get a very simple $\mathcal{M}$ (for example, $M{=}E$ and $\mathcal{M}(M, i){=}M[i]$) or we can keep $M$ very small at the cost of complicated $\mathcal{M}$ (for example $M$ = [0,1] and $\mathcal{M}$ combines these bits to get bit representation of $E$). An effective PSS for embedding table $E \in R^{n\times d}$ would have $(|M| + |\mathcal{M}|) \ll nd$. In context of training a PSS in an end-to-end manner, we additionally want (1) $\mathcal{M}$ is differentiable (2) $\mathcal{M}(M, i) $ computation is cheap. 

One of the key questions we want to answer is how small a PSS can we create for a ground truth embedding table $E$? i.e. how small can $(|M| + |\mathcal{M}|)$ be. For an arbitrary $E$, it is impossible to encode $E$ in sublinear memory accurately. So, in the next section, we define a $(1 \pm \epsilon)$ PSS 


\ssubsection{$(1 \pm \epsilon)$ parameter sharing setup for table $E$}
While creating an alternative representation of embedding table $E$, let us see what we would like to achieve. First, we would like to preserve all the pairwise inner products between two embeddings of tokens from the embedding table. Additionally, we also want to minimally affect the subsequent computations performed on the embeddings retrieved from the embedding table. In recommendation models, the embeddings retrieved are passed through neural network layers, which will perform inner products on embeddings. With this in mind, we propose the following definition of approximate PSS for embedding the table. 

\begin{definition}[$\mathbf{(1 \pm \epsilon)}$ \textbf{parameter sharing setup}]
A PSS $(M, \mathcal{M})$ is an $(1 \pm \epsilon)$ PSS for an embedding table E if the following holds, 
\begin{align*}
    & \forall x \in R^d \; \forall i \in \{ 0, .., n{-}1\} \\
    & \quad |\langle \mathcal{M}(M, i), x \rangle - \langle E[i], x \rangle | \leq \epsilon ||x||_2
\end{align*}
\end{definition}
We note that the $(1 \pm \epsilon)$ PSS for embedding table E also preserves the pairwise distances between the embeddings of different tokens, as is mentioned in the following result.
\begin{theorem} \label{thm:1}
Let  $(M, \mathcal{M})$ is a $(1 \pm \epsilon)$ PSS for embedding table E, we have ,
\begin{align*}
    & \forall i, j \in \{0,...,n{-}1\}, \\
    &| \;\; || \mathcal{M}(M, i) - \mathcal{M}(M, j) ||_2 - || (E[i] - E[j]) ||_2 \; \; | \leq 2\epsilon
\end{align*}
\end{theorem}
\begin{proof}
Presented in the appendix. We first analyze the norm of recovered embedding from PSS and then extend the result to pairwise distances.
\end{proof}
Naive application of JL-Lemma will give us an existence proof of memory M of size $n\log(n)$. In the next section, we show the existence of M with size $dlog(n)$, which is much smaller. Also, note that the condition for $(1 \pm \epsilon)$ PSS is quite different from that of l2-subspace embedding \cite{woodruff}. As we shall see in the proof, this leads to a logarithmic dependence on n (whereas l2-subspace embeddings for $A \in R^{n \times d}$ has no dependence on n) Also, as we shall see, the considerations in $(1 \pm \epsilon)$ PSS are different to those previously discussed in the literature for linear algebra problems.
\ssubsection{$(1 \pm \epsilon)$ PSS via Johnson–Lindenstrauss Transforms}
In this section, we answer the question of how small the $|M|$ and $|\mathcal{M}|$ can be. In order to do this, we give a construction of a $(1 \pm \epsilon)$ PSS using the Johnson–Lindenstrauss Transforms(JLT) matrices. We will use the following definition of JLT matrices from \cite{woodruff}
\begin{definition}
A randomly generated matrix $S \in R^{k \times n}$ is a $JLT(\epsilon, \delta, f)$ if with probability at least $(1 - \delta)$, for any $f$-element subset $V \subset R^n$ the following holds,
\begin{equation*}
    \forall v_1, v_2 \in V \; |\langle S v_1, S v_2 \rangle | \leq \epsilon ||v_1||_2 \, ||v_2||_2
\end{equation*}
\end{definition}
The construction of $(1 \pm \epsilon) $ PSS using JLT is presented in the following theorem,
\begin{figure}
    \caption{An i
    illustration of JLT based setup of PSS, the computation involving E is performed in the sketched-space using much lesser memory of learnable parameters M. The sketching is achieved using sketching matrix S}
    \centering
    \includegraphics[scale=0.2]{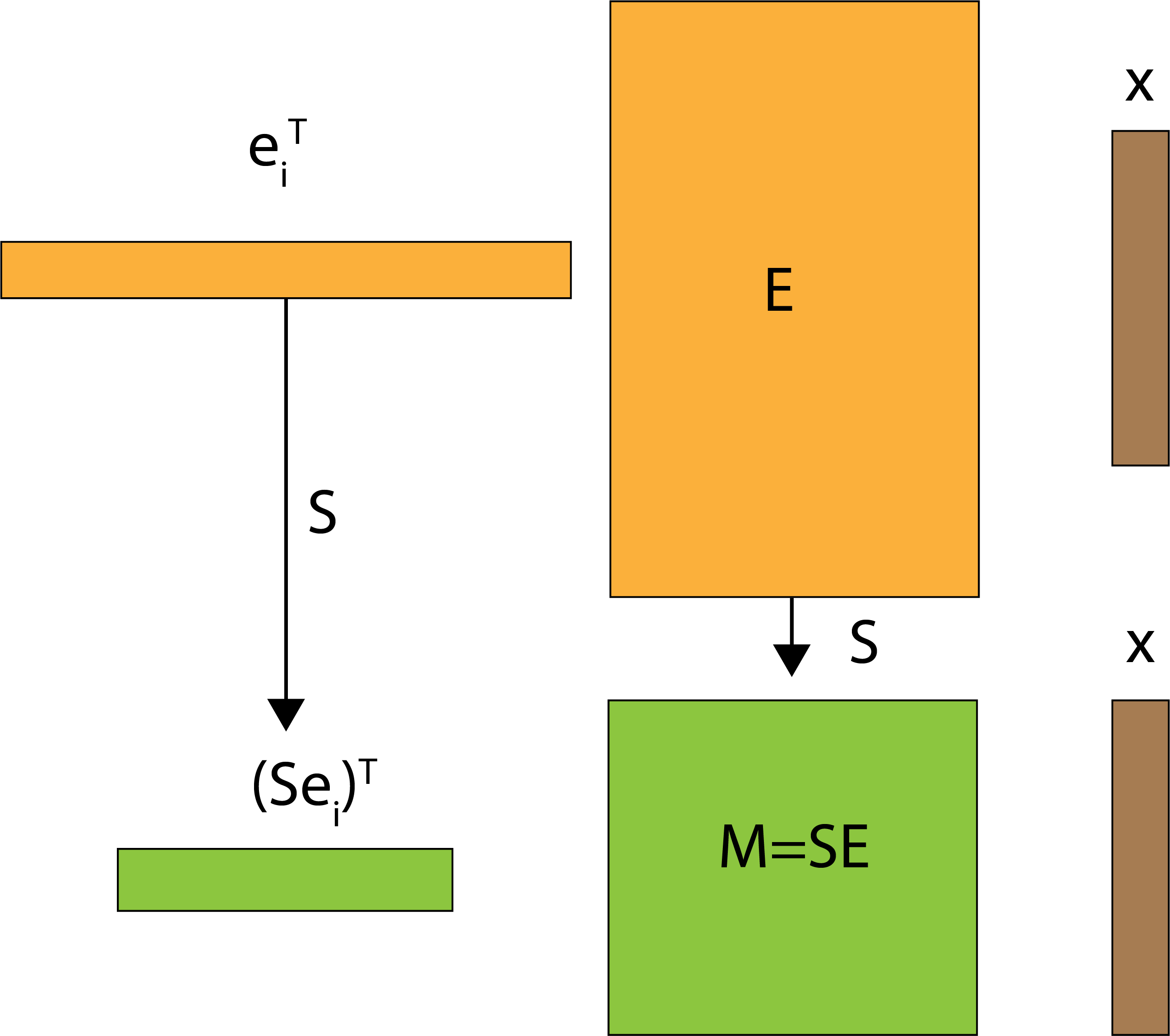}
    \label{fig:pss}
    \vspace{-0.6cm}
\end{figure}

\begin{theorem} \label{thm:2}
Let the embedding table be $E \in R^{n \times d}$. Consider a matrix $S$ which is $JLT( \epsilon / \sigma(E), \delta, 9^d + n))$ where $\sigma(E)$ denotes the maximum singular value of E. Then $M$ and $\mathcal{M}$ defined by 
\begin{align*}
    & M{=}S E \quad \mathcal{M}(M, i){=}(S e_i)^\top (M)
\end{align*}
is a $(1 \pm \epsilon) PSS$ with probability $(1-\delta)$
where $e_i \in R^n$ is a one-hot encoding of integer $i$ (i.e $e_i[i] = 1$ and rest all elements of $e_i$ are 0)
\end{theorem}
\begin{proof}
The complete proof is presented in the appendix. The intuition is that we want to maintain the computations of the form $e_i^\top E x$, which can be seen as an inner product between $e_i \in R^n$ and $Ex \in R^n$. So the inner products between vectors of sets $A = \{e_0, ..., e_{n-1}\}$ and column space of $E$ need to be preserved. Once we observe this, we apply JLT to the set of n discrete points and the $1/2-net$ of column space of $E$ which contains less than $9^d$ points, and get our result.
\end{proof}

Note that when using a JLT matrix S as described in theorem \ref{thm:2}, the storage cost of memory M is  $|M| = |SE| = kd$. For the mapping $\mathcal{M}$, we only need to store the matrix $S$. Thus $|\mathcal{M}| = |S|$. The cost of storing S will depend on the exact choice of JLT. Since $(Se_i)$ is a $\mathcal{O}(1)$ operation as $e_i$ is a one-hot encoded vector, the cost of $\mathcal{M}(M,i)$ is just that of vector-matrix multiplication between $(Se_i)$ and $M$. Again, the exact cost depends on the choice of S. Using the random matrix formed with i.i.d normal variables, the most basic form of JLT, we can give the following bound on the memory of learnable parameters $M$

\begin{corollary}
Let E be the embedding table under consideration. Also, let $S \in R^{k \times n}$ matrix where each element of S is i.i.d $\mathcal{N}(0, \frac{1}{k})$ with k = $\Omega(\sigma^2(E) / \epsilon^2 ( log(1/\delta) + \max(d, log(n)) )$, then with a probability $(1 - \delta)$, $(M=SE, \mathcal{M}(M, i) = (Se_i)^\top (M))$ is a $(1\pm\epsilon)$ PSS for E.
\end{corollary}
Thus we obtain an upper bound on M. $|M|=\mathcal{O}(1/\epsilon^2 \max(d^2, d.log(n)) )$ when using the standard normal matrix as JLT for embedding tables with bounded singular values. In this case, the mapping storage cost is $|\mathcal{M}| = kn$ as the matrix S is dense. The cost of applying $\mathcal{M}$ will be $kd$ and hence, $\mathcal{O}(1/\epsilon^2 \max(d^2, d.log(n)))$. In the next section, we will see how to reduce the cost of storing $S$ in practice. This logarithmic dependence on n in the case of embedding tables with very large $n$ explains why we can achieve a very high compression of 10000$\times$ for the criteo-tb dataset.

\ssubsection{$(1 \pm\epsilon)$ PSS in practice}
\textbf{$(1 \pm \epsilon)$ PSS from a learned embedding table E}
The above discussion directly gives us an algorithm to compute a $(1 \pm \epsilon) $ of a trained embedding table $E$. We just need to compute $M=SE$ and while retrieving an embedding compute $\mathcal{M}(M,i) = (Se_i)^\top M$.

\textbf{Training end-to-end $(1 \pm \epsilon)$ PSS for embedding table E}
We can also directly train an $(1 \pm \epsilon)$ PSS in an end-to-end manner. The idea is to directly train the compressed $M$. Thus, we have a matrix of learnable weights $M \in R^{k \times d}$. Let us now look at how the forward and backward pass of embedding retrieval mapping $\mathcal{M}$ looks like
\begin{equation*}
    \mathrm{ forward(i) } = (Se_i)^\top M
\end{equation*}
The forward function takes an integer $i$ and returns a $R^{1\times d}$ array which is embedding of $i$.
\begin{equation*}
    \mathrm{ backward(i, \Delta) } =   (Se_i) \Delta
\end{equation*}
The backward pass takes in all the arguments of the forward pass along with the $\Delta \in R^{1\times d}$ which is gradients of loss with respect to the output of the forward pass. We can back-propagate further to $W$ using the above formulation. The result of the backward pass is a $k \times d$ matrix. 

If we use a sparse $S$ such as with sparse JLT, we can achieve sparse gradients of W. That is, only a few entries of the gradient for W are non-zero. As we will see in section \ref{sec:robez_exp}, whether to propagate sparse or dense gradients is a vital implementation choice. We were able to achieve better training per iteration by exploiting this choice.

\textbf{Cost considerations for PSS using JLT.} In using PSS in end-to-end training, we would be learning the memory $M$ while keeping the mapping $\mathcal{M}$ constant. Thus, unlike in sketching for linear algebra \cite{woodruff}, we care about the costs of (1) storage of learnable parameters $|M|$, (2) Cost of storing mapping $|\mathcal{M}|$, and (3) cost of computing $\mathcal{M}(M, i)$. It is clear that using standard normal JLT is not feasible due to storage of matrix S will be more expensive than storing $E$ itself! There are a lot of sparse JLT$(\epsilon, \delta, f)$ matrices S which will reduce the storage cost of $S$ \cite{achlioptas2003database,dasgupta2010sparse,kane2014sparser}.  \cite{kane2014sparser} showed that the matrix needs to have a minimum column sparsity of $\Omega(\epsilon^{-1} log( f/ \delta) log( 1/ \epsilon))$ and hence the cost of storing S while using JLT matrices is lower bounded by $\Omega(n \epsilon^{-1} log( (n + 9^d)/ \delta) log( 1/ \epsilon))$ which can still be considerably high. In the next section, we see relaxations of conditions on S for practical purposes.
\ssubsection{Practical PSS and existing SOTA methods}
This section will briefly review practical relaxations of PSS defined above. 

\textbf{Sparse JL with independence relaxation:} The issue with storage costs of S can be alleviated in practice by relaxing the complete independence condition on entries of S. For example, the JLT matrix (which requires the same bound on k as given in theorem \ref{thm:2}) suggested by \cite{achlioptas2003database}, which selects each entry to be 0 with probability 2/3, $\pm 1$ with probability 1/6 independently can be generated on-the-fly using universal hashing. \cite{mitzenmacher2008simple} analyze why simpler hash functions can work well with data having enough entropy. The benefit of using universal hashing is that mapping $\mathcal{M}$ can now be stored in $O(1)$ memory. Thus using the \cite{achlioptas2003database} sparse matrix, we can obtain $(1 \pm \epsilon)$ PSS using total setup memory of  $|M| + |\mathcal{M}| =\mathcal{O}(1/\epsilon^2 \max(d^2, d.log(n)) )$ and cost of applying mapping $\mathcal{M}$ to be $\mathcal{O}(d)$ (hash computation and lookups) under valid independence relaxation.

Existing parameter-sharing based embedding compression methods can also be seen as PSS with varied distributions over S. We state a few of them below.
\begin{itemize}[nosep, leftmargin=*]
    \item \textbf{Hashing Trick} In this method, entire embeddings for a token $i$ is drawn from a randomly hashed location. In terms of PSS, we can define the mapping function $\mathcal{M}_h$ over a 2D matrix $\mathcal{M}_h \in R^{k \times d}$ where $k < n$.
    \begin{equation*}
    \mathcal{M}_h(M_h, i)[:] =  M_h[h(i), :]
    \end{equation*}
    where h is a hash function.
    \item \textbf{QR decomposition}\cite{QuoRemTrick19}. In this method, embedding for a token $i$ is drawn in chunks from separate memory vectors. In terms of PSS, we can define the mapping function $\mathcal{M}_q$ over, say l, pieces of 2D memory $M_1, M_2, .. M_l \in R^{k \times {d/l}}$ as follows
    \begin{equation}
        \mathcal{M}_q(\{M_1,M_2,..M_k\}, i)[j*(d/l):(j+1)*(d/l)] = M_j[h_j(i),:]
    \end{equation}
    In words, the $j^{th}$ chunk of $i^{th}$ embedding is recovered using the chunk at the location $h_j(i)$ in memory $M_r$

    \item \textbf{ HashedNet~\cite{hashtrick} and ROBE-Z~\cite{robez} }: In HashedNets, authors proposed mapping model weights randomly into a parameter array. ROBE-Z extended this idea by hashing chunks of embedding vector instead of individual elements. In terms of PSS, we can define the mapping function $\mathcal{M}_r$ over 1D memory $M_r$ as
    \begin{equation*}
    \mathcal{M}_r(M_r, i)[j*Z:(j+1)*Z] =     M_r[h(i,j):h(i,j)+Z]
    \end{equation*}
    In words, the $j^{th}$ chunk of $i^{th}$ embedding is recovered using the chunk at the location $h(i,j)$ from $M_r$. Here, $h:\mathbf{N}^2 \rightarrow \{0,..,|M_r|\}$ is a hash function. If we set $Z=1$, then we get the mapping function for HashedNet.
\end{itemize}
\begin{table}[]
\centering
\caption{State of the art in embedding compression on criteo datasets. (the ones in grey are not PSS ) }

\begin{tabular}{|l|l|l|l|}
\hline
Method                              & Dataset                                        & Compression                                       & Quality                               \\ \hline
HashingTrick                        & Crieto Kaggle                                  & 4$\times$                                         & worse                                 \\ \hline
QR Trick                            & Criteo Kaggle                                  & 4$\times$                                         & similar/slightly worse                \\ \hline
{\color[HTML]{C0C0C0} MD Embedding} & {\color[HTML]{C0C0C0} Criteo Kaggle}           & {\color[HTML]{C0C0C0} 16$\times$}                 & {\color[HTML]{C0C0C0} better/similar} \\ \hline
{\color[HTML]{C0C0C0} TT-Rec}       & {\color[HTML]{C0C0C0} Criteo kaggle/Criteo TB} & {\color[HTML]{C0C0C0} 112$\times$ / 117 $\times$} & {\color[HTML]{C0C0C0} better/simiar}  \\ \hline
ROBE                                & Criteo Kaggle/Criteo TB                        & 1000$\times$                                      & better/similar                        \\ \hline
\end{tabular}
\label{tab:sota}
\end{table}

We can summarize the state-of-the-art embedding compression in the table \ref{tab:sota}. Our PSS theory suggests that we should be able approximate the embedding table in memory of the order of $\mathcal{O}$(log(n). Thus, it is important to question if $1000\times$ compression is the best compression we can achieve on a dataset like CriteoTB where we have large ($\sim$400GB sized) embedding tables. As we shall see, one of the significant roadblocks in aiming to achieve higher compression is training time. The higher the compression, the more iterations are needed to train the model to a certain quality. In the next section, we discuss how to overcome this bottle-neck preventing training of highly compressed PSS.

\ssection{Road to 10000$\times$ compression with PSS}\label{sec:robez_exp}
Table~\ref{tab:kaggle_quality} shows the quality of embeddings for the criteo-kaggle dataset over varying compression rates across different popular recommendation models. We can see that the state-of-the-art compression of the criteo-kaggle dataset is $1000\times$. It is a consequence of PSS theory that larger embedding tables (prevalent in industrial scale recommendation models) such as those for criteo-tb should obtain more compression. Our goal is to study higher compression for large embedding tables and improve over existing SOTA ($1000$x). 

A significant roadblock in building highly compressed embedding tables is the slow convergence. Table \ref{tab:kaggconverge} shows that across different deep learning recommendation models, the number of iterations required to converge increases with higher compression. Specifically, for $1000\times$ compression, it takes up to $4\times$ iterations, and for models with $10000\times$ compression, most models do not converge in 15 epochs. If a model does not converge, it is hard to judge if the convergence is too slow or its capacity is consumed. 
In order to be able to train for larger iterations, we should be able to train the highly compressed models faster. Indeed, we expect some system benefits with highly compressed models, as shown in figure \ref{fig:tradeoff}. We find that existing methods cannot train highly compressed embedding tables because of the implementation choices. We propose the following implementation choices.

\textbf{Hashing chunks is better than hashing elements:} As suggested in \cite{robez}, we find that hashing chunks instead of individual elements is indeed helpful in reducing the latency of PSS systems. As shown in table \ref{tab:tt1}(d,e), both forward and backward passes are affected a lot by chunk sizes. It is also interesting to note that the effect of chunk sizes diminishes after size 16. Hence, we can choose a chunk size of 32 for hashing.
\begin{figure*}[]
\caption{These experiments are run on a single GPU for a simple embedding lookup and loss is taken to be sum of all retrieved elements with a batch size of 10240 and embedding dimension=128.  {\color{red} $\blacksquare$} higher latency (worse) {\color{green} $\blacksquare$} lower latency (better)}
\label{fig:smallt}

\centering
\resizebox{0.31\linewidth}{!}{
\subfloat[Fixed chunk size (16)]{
\begin{tabular}{|crrrr|}
\hline
\multicolumn{5}{|c|}{\textbf{Forward pass (ms) {[}sparse=false{]}}}                                                                                                                                                                                                                                                                                                    \\ \hline
\rowcolor[HTML]{FFFFFF} 
\multicolumn{1}{|c|}{\cellcolor[HTML]{FFFFFF}\textit{}}  & \multicolumn{4}{c|}{\cellcolor[HTML]{FFFFFF}\textbf{compression}}                                                                                                                                                                                                                                           \\ \hline
\rowcolor[HTML]{FFFFFF} 
\multicolumn{1}{|c|}{\cellcolor[HTML]{FFFFFF}\textbf{n}} & \multicolumn{1}{c|}{\cellcolor[HTML]{FFFFFF}{\color[HTML]{000000} 10x}} & \multicolumn{1}{c|}{\cellcolor[HTML]{FFFFFF}{\color[HTML]{000000} 100x}} & \multicolumn{1}{c|}{\cellcolor[HTML]{FFFFFF}{\color[HTML]{000000} 1000x}} & \multicolumn{1}{c|}{\cellcolor[HTML]{FFFFFF}{\color[HTML]{000000} 10000x}} \\ \hline
\multicolumn{1}{|c|}{\cellcolor[HTML]{FFFFFF}4M}         & \multicolumn{1}{r|}{\cellcolor[HTML]{FBC768}0.40}                       & \multicolumn{1}{r|}{\cellcolor[HTML]{B0C977}0.31}                        & \multicolumn{1}{r|}{\cellcolor[HTML]{7EC182}0.27}                         & \cellcolor[HTML]{70BF85}0.27                                               \\ \hline
\multicolumn{1}{|c|}{\cellcolor[HTML]{FFFFFF}8M}         & \multicolumn{1}{r|}{\cellcolor[HTML]{FFD666}0.36}                       & \multicolumn{1}{r|}{\cellcolor[HTML]{BACA75}0.32}                        & \multicolumn{1}{r|}{\cellcolor[HTML]{92C47E}0.29}                         & \cellcolor[HTML]{76C084}0.27                                               \\
\hline
\end{tabular}
}\label{tab:tt1}
}
\resizebox{0.31\linewidth}{!}{
\subfloat[Fixed chunk size (16)]{\begin{tabular}{|crrrr|}
\hline
\multicolumn{5}{|c|}{\textbf{Backward pass (ms) {[}sparse=false{]}}}                                                                                                                                                                                                                                                                                                    \\ \hline
\rowcolor[HTML]{FFFFFF} 
\multicolumn{1}{|c|}{\cellcolor[HTML]{FFFFFF}\textit{}}  & \multicolumn{4}{c|}{\cellcolor[HTML]{FFFFFF}\textbf{compression}}                                                                                                                                                                                                                                           \\ \hline
\rowcolor[HTML]{FFFFFF} 
\multicolumn{1}{|c|}{\cellcolor[HTML]{FFFFFF}\textbf{n}} & \multicolumn{1}{c|}{\cellcolor[HTML]{FFFFFF}{\color[HTML]{000000} 10x}} & \multicolumn{1}{c|}{\cellcolor[HTML]{FFFFFF}{\color[HTML]{000000} 100x}} & \multicolumn{1}{c|}{\cellcolor[HTML]{FFFFFF}{\color[HTML]{000000} 1000x}} & \multicolumn{1}{c|}{\cellcolor[HTML]{FFFFFF}{\color[HTML]{000000} 10000x}} \\ \hline
\multicolumn{1}{|c|}{\cellcolor[HTML]{FFFFFF}4M}         & \multicolumn{1}{r|}{\cellcolor[HTML]{FFD666}2.55}                       & \multicolumn{1}{r|}{\cellcolor[HTML]{6ABE86}0.59}                        & \multicolumn{1}{r|}{\cellcolor[HTML]{57BB8A}0.35}                         & \cellcolor[HTML]{57BB8A}0.35                                               \\ \hline
\multicolumn{1}{|c|}{\cellcolor[HTML]{FFFFFF}8M}         & \multicolumn{1}{r|}{\cellcolor[HTML]{FFD366}4.81}                       & \multicolumn{1}{r|}{\cellcolor[HTML]{7CC182}0.83}                        & \multicolumn{1}{r|}{\cellcolor[HTML]{58BB8A}0.37}                         & \cellcolor[HTML]{57BB8A}0.35                                               \\
\hline
\end{tabular}}
}
\resizebox{0.31\linewidth}{!}{
\subfloat[Fixed chunk size (16)]{\begin{tabular}{|crrrr|}
\hline
\multicolumn{5}{|c|}{\textbf{Backward pass (ms) {[}sparse=true{]}}}                                                                                                                                                                                                                                                                                                   \\ \hline
\rowcolor[HTML]{FFFFFF} 
\multicolumn{1}{|c|}{\cellcolor[HTML]{FFFFFF}\textit{}}  & \multicolumn{4}{c|}{\cellcolor[HTML]{FFFFFF}\textbf{compression}}                                                                                                                                                                                                                                           \\ \hline
\rowcolor[HTML]{FFFFFF} 
\multicolumn{1}{|c|}{\cellcolor[HTML]{FFFFFF}\textbf{n}} & \multicolumn{1}{c|}{\cellcolor[HTML]{FFFFFF}{\color[HTML]{000000} 10x}} & \multicolumn{1}{c|}{\cellcolor[HTML]{FFFFFF}{\color[HTML]{000000} 100x}} & \multicolumn{1}{c|}{\cellcolor[HTML]{FFFFFF}{\color[HTML]{000000} 1000x}} & \multicolumn{1}{c|}{\cellcolor[HTML]{FFFFFF}{\color[HTML]{000000} 10000x}} \\ \hline
\multicolumn{1}{|c|}{\cellcolor[HTML]{FFFFFF}4M}         & \multicolumn{1}{r|}{\cellcolor[HTML]{FFD666}2.67}                       & \multicolumn{1}{r|}{\cellcolor[HTML]{FBD567}2.48}                        & \multicolumn{1}{r|}{\cellcolor[HTML]{FFD666}2.56}                         & \cellcolor[HTML]{FED567}2.52                                               \\ \hline
\rowcolor[HTML]{FFD666} 
\multicolumn{1}{|c|}{\cellcolor[HTML]{FFFFFF}8M}         & \multicolumn{1}{r|}{\cellcolor[HTML]{FFD666}2.53}                       & \multicolumn{1}{r|}{\cellcolor[HTML]{FFD666}2.58}                        & \multicolumn{1}{r|}{\cellcolor[HTML]{FBD567}2.48}                         & 2.53                                                                       \\
\hline
\end{tabular}}
}
\centering
\resizebox{0.31\linewidth}{!}{
\subfloat[Fixed compression 1000$\times$]{\begin{tabular}{|crrrrrr|}
\hline
\multicolumn{7}{|c|}{\textbf{Forward pass (ms) {[}sparse=false{]}}}                                                                                                                                                                                                                                                                                                                                                                         \\ \hline
\rowcolor[HTML]{FFFFFF} 
\multicolumn{1}{|c|}{\cellcolor[HTML]{FFFFFF}\textit{}}  & \multicolumn{6}{c|}{\cellcolor[HTML]{FFFFFF}\textbf{chunk size}}                                                                                                                                                                                                                                                                                                          \\ \hline
\multicolumn{1}{|c|}{\cellcolor[HTML]{FFFFFF}\textbf{n}} & \multicolumn{1}{c|}{\cellcolor[HTML]{FFFFFF}{\color[HTML]{000000} 1}} & \multicolumn{1}{c|}{\cellcolor[HTML]{FFFFFF}{\color[HTML]{000000} 2}} & \multicolumn{1}{c|}{\cellcolor[HTML]{FFFFFF}{\color[HTML]{000000} 4}} & \multicolumn{1}{c|}{\cellcolor[HTML]{FFFFFF}{\color[HTML]{000000} 8}} & \multicolumn{1}{l|}{16}                           & \multicolumn{1}{l|}{32}      \\ \hline
\multicolumn{1}{|c|}{\cellcolor[HTML]{FFFFFF}4M}         & \multicolumn{1}{r|}{\cellcolor[HTML]{FAC169}0.68}                     & \multicolumn{1}{r|}{\cellcolor[HTML]{F7D468}0.58}                     & \multicolumn{1}{r|}{\cellcolor[HTML]{E5D16C}0.55}                     & \multicolumn{1}{r|}{\cellcolor[HTML]{7CC182}0.32}                     & \multicolumn{1}{r|}{\cellcolor[HTML]{79C083}0.31} & \cellcolor[HTML]{7AC083}0.31 \\
\hline
\end{tabular}
}
}
\resizebox{0.31\linewidth}{!}{
\subfloat[Fixed compression 1000$\times$]{\begin{tabular}{|crrrrrr|}
\hline
\multicolumn{7}{|c|}{\textbf{Backward pass (ms) {[}sparse=false{]}}}                                                                                                                                                                                                                                                                                                                                                                                                       \\ \hline
\rowcolor[HTML]{FFFFFF} 
\multicolumn{1}{|c|}{\cellcolor[HTML]{FFFFFF}\textit{}}           & \multicolumn{6}{c|}{\cellcolor[HTML]{FFFFFF}\textbf{ chunk size}}                                                                                                                                                                                                                                                                                                                                \\ \hline
\rowcolor[HTML]{8093B3} 
\multicolumn{1}{|l|}{\cellcolor[HTML]{DFE4EC}\textit{\textbf{n}}} & \multicolumn{1}{r|}{\cellcolor[HTML]{8093B3}{\color[HTML]{FFFFFF} 1}} & \multicolumn{1}{r|}{\cellcolor[HTML]{8093B3}{\color[HTML]{FFFFFF} 2}} & \multicolumn{1}{r|}{\cellcolor[HTML]{8093B3}{\color[HTML]{FFFFFF} 4}} & \multicolumn{1}{r|}{\cellcolor[HTML]{8093B3}{\color[HTML]{FFFFFF} 8}} & \multicolumn{1}{r|}{\cellcolor[HTML]{8093B3}{\color[HTML]{FFFFFF} 16}} & {\color[HTML]{FFFFFF} 32}     \\ \hline
\multicolumn{1}{|c|}{\cellcolor[HTML]{FFFFFF}4M}                  & \multicolumn{1}{r|}{\cellcolor[HTML]{91C47E}1.29}                     & \multicolumn{1}{r|}{\cellcolor[HTML]{81C181}1.10}                     & \multicolumn{1}{r|}{\cellcolor[HTML]{69BE86}0.80}                     & \multicolumn{1}{r|}{\cellcolor[HTML]{5CBB89}0.64}                     & \multicolumn{1}{r|}{\cellcolor[HTML]{57BB8A}0.58}                      & \cellcolor[HTML]{57BB8A}0.57  \\
\hline
\end{tabular}}
}
\resizebox{0.31\linewidth}{!}{
\subfloat[Fixed compression 1000$\times$]{\begin{tabular}{|crrrrrr|}
\hline
\multicolumn{7}{|c|}{\textbf{Backward pass (ms) {[}sparse=true{]}}}                                                                                                                                                                                                                                                                                                                                                                                                       \\ \hline
\rowcolor[HTML]{FFFFFF} 
\multicolumn{1}{|c|}{\cellcolor[HTML]{FFFFFF}\textit{}}           & \multicolumn{6}{c|}{\cellcolor[HTML]{FFFFFF}\textbf{chunk size}}                                                                                                                                                                                                                                                                                                                               \\ \hline
\rowcolor[HTML]{8093B3} 
\multicolumn{1}{|l|}{\cellcolor[HTML]{DFE4EC}\textit{\textbf{n}}} & \multicolumn{1}{r|}{\cellcolor[HTML]{8093B3}{\color[HTML]{FFFFFF} 1}} & \multicolumn{1}{r|}{\cellcolor[HTML]{8093B3}{\color[HTML]{FFFFFF} 2}} & \multicolumn{1}{r|}{\cellcolor[HTML]{8093B3}{\color[HTML]{FFFFFF} 4}} & \multicolumn{1}{r|}{\cellcolor[HTML]{8093B3}{\color[HTML]{FFFFFF} 8}} & \multicolumn{1}{r|}{\cellcolor[HTML]{8093B3}{\color[HTML]{FFFFFF} 16}} & {\color[HTML]{FFFFFF} 32}    \\ \hline
\multicolumn{1}{|c|}{\cellcolor[HTML]{FFFFFF}4M}                  & \multicolumn{1}{r|}{\cellcolor[HTML]{FED366}3.17}                     & \multicolumn{1}{r|}{\cellcolor[HTML]{FFD466}3.04}                     & \multicolumn{1}{r|}{\cellcolor[HTML]{FCD567}2.61}                     & \multicolumn{1}{r|}{\cellcolor[HTML]{F7D468}2.54}                     & \multicolumn{1}{r|}{\cellcolor[HTML]{F3D469}2.50}                      & \cellcolor[HTML]{F1D369}2.48 \\
\hline
\end{tabular}
}
}

\end{figure*}

\textbf{Dense gradients are suitable for highly compressed PSS} Generally, in embedding tables, implementations use sparse gradients. Dense gradients are particularly wasteful in embedding tables, where in each iteration, only a few weights are involved in the computation. However, the same does not apply to highly compressed PSS. In the case of PSS, the algorithm for sparse gradient computation is more involved than that for dense gradients. The algorithms for sparse and dense gradient computation are specified in algorithm \ref{algo:gradient}. Also, the scatter operation is much faster when the memory size of embeddings is smaller. Hence, dense gradients work well with higher compressions. From table \ref{tab:tt1}(b,c), we can see that at higher compression rates, the backward pass is up to $8\times$ faster with dense gradients as compared to sparse gradients.
\begin{algorithm}

\caption{Backward Pass for PSS}\label{algo:gradient}
\begin{algorithmic}

\Require $\mathcal{G}_O\in R^{B \times d}$, $\mathcal{I} \in R^{B \times d}$, $\mathrm{sparse} \in \{\mathrm{True, False}\}$, M: compressed memory.

\Comment{$\mathcal{G}_O$ is the gradient w.r.t output, and $\mathcal{I}$ is the mapping matrix showing the locations from which the output was accessed}

\Ensure $\mathcal{G}_i$ is gradient w.r.t input. 
\If{$\mathrm{sparse}$}
    \State $\mathrm{unique, invIdx} \leftarrow \mathrm{findUnique}(\mathcal{I}.\mathrm{flatten}(), \mathrm{returnInverse}=\mathrm{True})$
    \State $g \leftarrow \mathrm{Tensor((unique.shape,))}$
    \State $g.\mathrm{scatter\_add}(\mathrm{invIdx},\mathcal{G}_o.\mathrm{flatten}()$
    \State $\mathcal{G}_i = \mathrm{sparseTensor(location=unique, values}=\mathrm{g, size=M.size})$
\ElsIf{$\mathrm{dense}$}
    \State $\mathcal{G}_i \leftarrow \mathrm{Tensor(M.size)}$
    \State $\mathcal{G}_i.\mathrm{scatter\_add}(0, \mathcal{I}.\mathrm{flatten}(), \mathcal{G}_o.\mathrm{flatten}())$
\EndIf
\end{algorithmic}
\end{algorithm}

\textbf{Forward/Backward kernel optimizations} Latency for PSS, and CUDA kernels in general,  is very sensitive to the usage of shared memory, occupancy, and communication between CPU and GPUs. We also optimize our PSS code to minimize the data movement costs, implement custom kernels to fuse operations to improve shared memory usage and optimize CUDA grid sizes to obtain the best performance.

With the implementation choices mentioned above, we are ready to train a $10000\times$ compressed PSS for criteo-tb. We implement our PSS using ROBE-style hashing.

\ssection{Experimental results on Criteo datasets}  \label{sec:results}
We perform experiments on criteo-kaggle and criteo-tb datasets in order to confirm the following hypothesis,
\begin{enumerate}[nosep, leftmargin=*]
\item Theory in section \ref{sec:theory} dictates that embedding table ($E \in R^{n \times d}$) can be represented in memory logarithmic in $n$. High compression should be possible in these datasets. Specifically, higher compression should be possible in larger embedding tables.
\item The system advantage of smaller PSS should compensate for the convergence advantage of the original model.
\end{enumerate}
\textbf{Datasets:} Criteo-kaggle and criteo-tb datasets have 13 integer and 26 categorical features. criteo-kaggle data was collected over seven days, whereas the criteo-tb dataset was collected over 23 days. criteo-tb is one of the largest recommendation datasets in the public domains with around 800 million token values making the embedding tables of size around 400GB. (with d = 128). criteo-kaggle is smaller and has a 2GB sized embedding table. Note that industrial-scale models are much larger than the DLRM model we talk about here. For example, Facebook recently published a model sized 50TB \cite{mudigere2021high}. One can extrapolate the benefits of PSS to industrial-scale models based on this case study. 

\textbf{Models:} Facebook MLPerf DLRM\cite{DLRM19} model, available under Apache-2.0 license, for the criteo-tb dataset achieves the target AUC (0.8025) with the embedding memory of around 100GB. This model uses a maximum cap of 40M indices per embedding table, leading to a total of 204M embeddings. This model cannot be trained on a single GPU (like V100) and is trained using a multiple GPUs (4 or 8). For the criteo-tb dataset, we use the DLRM MLPerf model. For criteo-kaggle dataset, we use an array of state-of-the-art models DLRM\cite{DLRM19}, DCN\cite{DCN17}, AUTOINT\cite{song2019autoint}, DEEPFM \cite{guo2017deepfm}, XDEEPFM \cite{lian2018xdeepfm} and FIBINET. We use our PSS implementation described in \ref{sec:robez_exp} as a compressed model.
\begin{figure*}[]
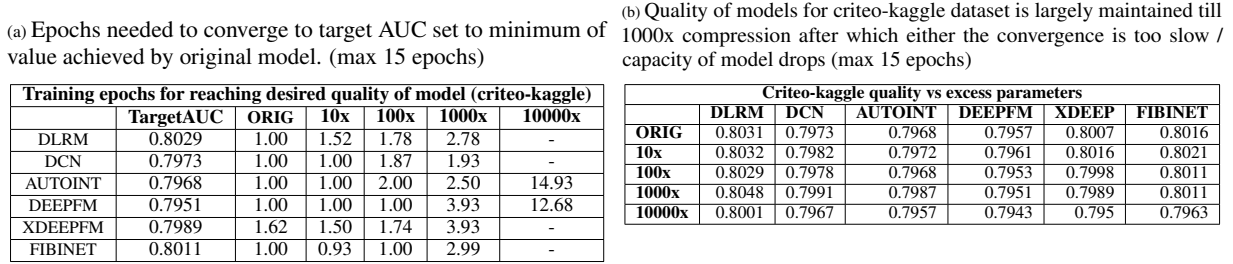

\caption{Quality and convergence on criteo-kaggle of 5 popular models vs. compression. The standard deviation of all AUC results is within 0.0009.}
\label{}
\centering
\resizebox{0.49\linewidth}{!}{
\subfloat[\large{Epochs needed to converge to target AUC set to minimum of value achieved by original model. (max 15 epochs)}]{
\begin{tabular}{|ccccccc|}
\hline
\multicolumn{7}{|c|}{\textbf{{Training epochs for reaching desired quality of model (criteo-kaggle)}}}                                                                                                                                                                                                                                                                                                                                                      \\ \hline
\multicolumn{1}{|c|}{}        & \multicolumn{1}{c|}{\textbf{{TargetAUC}}} & \multicolumn{1}{c|}{\textbf{\small{ORIG}}} & \multicolumn{1}{c|}{\textbf{\begin{tabular}[c]{@{}c@{}}10x\end{tabular}}} & \multicolumn{1}{c|}{\textbf{\begin{tabular}[c]{@{}c@{}}100x\end{tabular}}} & \multicolumn{1}{c|}{\textbf{\begin{tabular}[c]{@{}c@{}} 1000x\end{tabular}}} & \textbf{\begin{tabular}[c]{@{}c@{}}10000x\end{tabular}} \\ \hline
\multicolumn{1}{|c|}{\small{DLRM}}    & \multicolumn{1}{c|}{0.8029}              & \multicolumn{1}{c|}{1.00}          & \multicolumn{1}{c|}{1.52}                                                          & \multicolumn{1}{c|}{1.78}                                                           & \multicolumn{1}{c|}{2.78}                                                            & -                                                                \\ \hline
\multicolumn{1}{|c|}{\small{DCN}}     & \multicolumn{1}{c|}{0.7973}              & \multicolumn{1}{c|}{1.00}          & \multicolumn{1}{c|}{1.00}                                                          & \multicolumn{1}{c|}{1.87}                                                           & \multicolumn{1}{c|}{1.93}                                                            & -                                                                \\ \hline
\multicolumn{1}{|c|}{\small{AUTOINT}} & \multicolumn{1}{c|}{0.7968}              & \multicolumn{1}{c|}{1.00}          & \multicolumn{1}{c|}{1.00}                                                          & \multicolumn{1}{c|}{2.00}                                                           & \multicolumn{1}{c|}{2.50}                                                            & 14.93                                                            \\ \hline
\multicolumn{1}{|c|}{\small{DEEPFM}}  & \multicolumn{1}{c|}{0.7951}              & \multicolumn{1}{c|}{1.00}          & \multicolumn{1}{c|}{1.00}                                                          & \multicolumn{1}{c|}{1.00}                                                           & \multicolumn{1}{c|}{3.93}                                                            & 12.68                                                            \\ \hline
\multicolumn{1}{|c|}{\small{XDEEPFM}} & \multicolumn{1}{c|}{0.7989}              & \multicolumn{1}{c|}{1.62}          & \multicolumn{1}{c|}{1.50}                                                          & \multicolumn{1}{c|}{1.74}                                                           & \multicolumn{1}{c|}{3.93}                                                            & -                                                                \\ \hline
\multicolumn{1}{|c|}{\small{FIBINET}} & \multicolumn{1}{c|}{0.8011}              & \multicolumn{1}{c|}{1.00}          & \multicolumn{1}{c|}{0.93}                                                          & \multicolumn{1}{c|}{1.00}                                                           & \multicolumn{1}{c|}{2.99}                                                            & -                                                                \\ \hline
\end{tabular}}\label{tab:kaggconverge}
}
\resizebox{0.49\linewidth}{!}{
\subfloat[\large{Quality of models for criteo-kaggle dataset is largely maintained till 1000x compression after which either the convergence is too slow / capacity of model drops (max 15 epochs)}]{

\begin{tabular}{|lrrrrrr|}
\hline
\multicolumn{7}{|c|}{\textbf{Criteo-kaggle quality vs excess parameters}}                                                                                                                                                                                                          \\ \hline
\multicolumn{1}{|l|}{}                & \multicolumn{1}{l|}{\textbf{{DLRM}}} & \multicolumn{1}{l|}{\textbf{{DCN}}} & \multicolumn{1}{l|}{\textbf{{AUTOINT}}} & \multicolumn{1}{l|}{\textbf{{DEEPFM}}} & \multicolumn{1}{l|}{\textbf{{XDEEP}}} & \multicolumn{1}{l|}{\textbf{{FIBINET}}} \\ \hline
\multicolumn{1}{|l|}{\textbf{ORIG}}   & \multicolumn{1}{r|}{0.8031}        & \multicolumn{1}{r|}{0.7973}       & \multicolumn{1}{r|}{0.7968}           & \multicolumn{1}{r|}{0.7957}          & \multicolumn{1}{r|}{0.8007}         & 0.8016                                \\ \hline
\multicolumn{1}{|l|}{\textbf{10x}}    & \multicolumn{1}{r|}{0.8032}        & \multicolumn{1}{r|}{0.7982}       & \multicolumn{1}{r|}{0.7972}           & \multicolumn{1}{r|}{0.7961}          & \multicolumn{1}{r|}{0.8016}         & 0.8021                                \\ \hline
\multicolumn{1}{|l|}{\textbf{100x}}   & \multicolumn{1}{r|}{0.8029}        & \multicolumn{1}{r|}{0.7978}       & \multicolumn{1}{r|}{0.7968}           & \multicolumn{1}{r|}{0.7953}          & \multicolumn{1}{r|}{0.7998}         & 0.8011                                \\ \hline
\multicolumn{1}{|l|}{\textbf{1000x}}  & \multicolumn{1}{r|}{0.8048}        & \multicolumn{1}{r|}{0.7991}       & \multicolumn{1}{r|}{0.7987}           & \multicolumn{1}{r|}{0.7951}          & \multicolumn{1}{r|}{0.7989}         & 0.8011                                \\ \hline
\multicolumn{1}{|l|}{\textbf{10000x}} & \multicolumn{1}{r|}{0.8001}        & \multicolumn{1}{r|}{0.7967}       & \multicolumn{1}{r|}{0.7957}           & \multicolumn{1}{r|}{0.7943}          & \multicolumn{1}{r|}{0.795}          & 0.7963                                \\ \hline
\end{tabular}
} \label{tab:kaggle_quality}
}
\vspace{-0.5cm}
\end{figure*}
\begin{figure*}[]
\caption{Quality and convergence on criteo-tb with DLRM vs. compression. The standard deviation of AUC results is within 0.0009. }\label{tab:tbconverge}\label{tab:tbquality} \label{tab:tbtables}
\label{}
\centering
\resizebox{0.49\linewidth}{!}{
\subfloat[\large{Epochs required to reach 0.8025 AUC for DLRM model on criteo-tb dataset. Relative time computes the ratio of A*B to original model}]{

\begin{tabular}{|c|c|c|c|c|c|}
\hline
    &                & orig & $10^2\times$ & $10^3\times$ & $10^4\times$ \\ \hline
A   & Epochs         & 1.00 & 1.94         & 1.9          & 4.35         \\ \hline
B   & Time /1000 itr & 50.6 & 31.04        & 29.6         & 11.6         \\ \hline
A*B & Relative Time  & 1    & 1.19         & 1.11         & 1.04         \\ \hline
\end{tabular}

}
}
\resizebox{0.49\linewidth}{!}{
\subfloat[\large{$10^4\times$ compression also reaches the target AUC. For $10^5\times$ convergence is too slow/mode capacity is reached. (max 15 epochs) }]{


\begin{tabular}{|c|c|c|c|c|c|}
\hline
 0.8025 AUC & orig & $10^2\times$ & $10^3\times$ & $10^4\times$ & $10^5\times$ \\ \hline
DLRM            & Yes  & Yes          & Yes          & Yes          & No          \\ \hline
\end{tabular}
} 
}
\vspace{-0.5cm}
\end{figure*}

\textbf{Quality of model vs. Excess parameters:}
Tables \ref{tab:kaggle_quality}(b) and \ref{tab:tbquality}(b) show the results for the two datasets across different values of compression. In table \ref{tab:kaggle_quality}(b), we can see that across different models, the quality of the model is maintained until 1000$\times$ compression. As criteo-tb embedding tables are much larger than the criteo-kaggle dataset, according to the section \ref{sec:theory}, we should see larger values of compression. Indeed, we obtain a 10000$\times$ compression without loss of quality of the model for criteo-tb. This level of compression is unprecedented in embedding compression literature. These experiments validate our first hypothesis and provide a new state-of-the-art embedding compression. In the rest of the section, we evaluate how the system advantage of PSS compares against the original embeddings in various aspects.

\textbf{Inference time} 
In figure \ref{fig:inf}, we compare the inference time of test data (89M samples) in the criteo-tb dataset with a batch size of 16384 with PSS on a single Quadro RTX-8000 GPU and the original model on 8 Quadro RTX-8000 GPUs. The total time for inference for the original model is around 203 seconds, whereas PSS is around 97 seconds. In the original model trained on 8GPUs, embedding lookup is model-parallel, whereas the rest of the computation ( bot-mlp, top-mlp, and interactions) is data-parallel. There is a steep improvement of over 3$\times$ in embedding time lookup as we move from a distributed GPU setup for embedding tables to a PSS on a single GPU. The extra computation time in PSS is much smaller than all-to-all communication costs in the original embedding table lookup. We include the time required for data distribution (initial) in bot-mlp and data-gather (final) in top-mlp timings. This communication cost is high, and we can see that we are better off performing the entire computation on a single GPU.

\textbf{Model training time per iteration vs excess parameters} Figure \ref{fig:tt} records the time-taken for 1000 iterations of training. While using embedding tables, we can choose to back-propagate sparse or dense embedding gradients. The general idea is to use sparse gradients when very few gradients are non-zeros. In the original model, we can only use sparse gradients as using dense gradients is prohibitive w.r.t to computation and communication. In PSS, however, we compare both modes for gradient back-propagation. We see that dense gradients perform exceptionally well at higher rates of compression. The performance of sparse gradients is constant across different compression rates as the workload is similar. Generally, it seems good to use dense gradients for PSS when the effective memory size (i.e., the final memory of compressed embedding tables ) is small. It is noteworthy that the training time per iteration reduces significantly with higher compression. For example, with 1000x compression, the training time is $1.7 \times$ lesser than the time taken by the original model, whereas with 10000x compression, the time per iteration is 4.37 $\times$ lesser.

\textbf{Model convergence and overall time vs. excess parameters}
We observe a consistent trend in convergence: as the number of excess parameters in the models reduces, the convergence becomes slower. This can be seen across models and datasets as shown in tables \ref{tab:kaggconverge}(a) and \ref{tab:tbconverge}(a).  As an example, table \ref{tab:tbconverge} shows that with 10000$\times$ compression, we require 4.55 epochs as compared to the original model's one epoch. This might seem unfavorable at first. However, as seen in the previous section, there is a significant gain in training time per iteration. As seen in table  \ref{tab:tbconverge}, this gain largely compensates for the disadvantage in terms of convergence, making the overall training times similar. For example, while we need 4.55 epochs with 10000$\times$ compressed PSS, the training time per iteration goes down by a factor of 4.37. Hence, the overall training time is only 1.04$\times$ original time. 

\textbf{Engineering challenges and costs vs. excess parameters:}
The excess parameters in a recommendation model primarily appear due to the construction of embedding tables. The industrial-scale embedding tables can go as large as hundreds of terabytes. With the increase in the model's size, the model needs to be distributed across different nodes and GPUs. The complexities of efficiently running a  distributed model training include many considerations such as non-uniform memory allocation and communication costs. The increasing literature detailing the engineering solutions to training such models is evidence of the fact that training such models require engineering ingenuity to solve the challenges involved \cite{mudigere2021high, kalamkar2020optimizing, jiang2019xdl, park2018deep, naumov2020deep, yang2020training}. For example, in \cite{mudigere2021high}, the authors detail their solution to train a model of size 50TB on a new distributed system. Similarly, in \cite{kalamkar2020optimizing}, authors talk about their optimizations of recommendation models on CPU clusters. In this paper, we argue that the large-scale nature of embedding-based recommendation models appears due to embedding tables where a complete $n \times d$ table with blown up $n$ does not add value to model capacity while adding significant engineering challenges. Distributed systems also imply significant energy costs. PSS can avoid these downsides by fitting entire learnable parameters on a single GPU/node.

\begin{figure}[]
\begin{subfigure}{0.45\textwidth}
    \centering
    \includegraphics[trim={0 20 0 0}, clip, width=\linewidth]{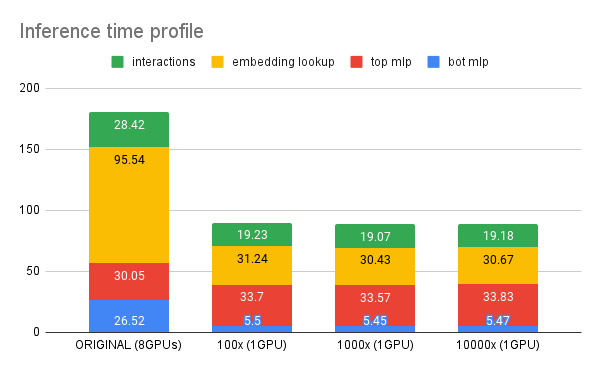}
    \caption{[criteo-tb] DLRM inference time taken for over 89M samples with a batch size of 16384. Communication causes embedding lookup for original model much slower than that for PSS models when entire model is located on a single GPU }

    \label{fig:inf}
\end{subfigure}
\begin{subfigure}{0.45 \textwidth}
\centering
    \includegraphics[trim={0 20 60 0}, clip, width=\linewidth]{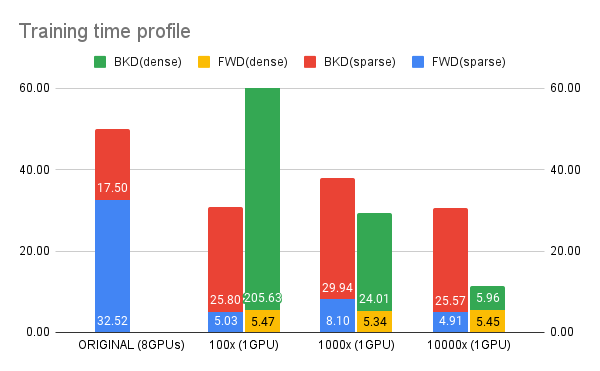}
        \caption{DLRM training time for 1000 iterations of training with a mini-batch size of 2048 and SGD optimizer}
        \label{fig:tt}

\end{subfigure}
\vspace{-0.5cm}
\end{figure}

\vspace{-0.2cm}
\ssection{Conclusion}
Today's prevalent idea in deep learning(DL) is overparameterized models with stochastic iterative learning. While this paradigm has led to success in DL, the sustainability of this route to success is a grim question \cite{thompson2021deep}. We evaluate the value of embedding-based excess parameters in DLRM models. We conclude that compressed DLRM models are better for fast inference and are easy to train and deploy as expected. More importantly, they can also be trained in the same overall time despite having a slower convergence rate. The critical observation is that compressed models enjoy a system advantage they can exploit, which reduces time per iteration. The paper also highlights that the correct way to compare the quality of models of highly different sizes is to run the models for equal time rather than equal iterations. 
\vspace{-0.2cm}

\bibliography{example_paper}
\bibliographystyle{unsrt}

\newpage
\appendix
\section{Appendix}
\section{Proof of Theorem \ref{thm:2}}
Let the embedding table be $E \in R^{n \times d}$. Consider a matrix $S$ which is $JLT( \epsilon / \sigma(E), \delta, 9^d + n))$ where $\sigma(E)$ denotes the maximum singular value of E. Then $M$ and $\mathcal{M}$ defined by 
\begin{align*}
    & M{=}S E \quad \mathcal{M}(M, i){=}(S e_i)^\top (M)
\end{align*}
is a $(1 \pm \epsilon) PSS$
where $e_i \in R^n$ is a one-hot encoding of integer $i$ (i.e $e_i[i] = 1$ and rest all elements of $e_i$ are 0)

\begin{proof}
Johnson Lindenstrauss transforms are a class of transformations which preserve pair wise inner products ( equivalently norms ) of a set of vectors in a space
Definition of JLT transform from \cite{woodruff} is 
A random matrix $S \in R^{k \times n}$ forms a $JLT(\epsilon, \delta, f)$ if with probability $(1-\delta)$ for any f-element subset $V \subset R^n$, for all $v_1, v_2$ it holds that 
\begin{equation}
    |\langle Sv_1, Sv_2 \rangle - \langle v_1, v_2 \rangle | \leq \epsilon ||v_1||_2 ||v_2||_2
\end{equation}

\paragraph{ Choice of S matrix and inner product preservation }

Consider the following sets of points. 
\begin{enumerate}
    \item $\{e_i\}_{i=0}^{n-1}$
    \item $1/2-net$ over the set $\{u | u \in \textrm{Column-space}(E) \; ||u|||_2=1\}$
\end{enumerate}
We use the lemma 5 from \cite{woodruff} to have the number of points in set 2 bounded by $9^d$.
Let S be a matrix that is JLT($\delta, \epsilon, (n + 9^d)$). For these $(n + 9^d)$ points with probability $(1-\delta)$, the matrix S preserves inner products as per definition of JLT given above. Let us consider 4 cases w.r.t to the following condition
\begin{equation}
        |\langle Sv_1, Sv_2 \rangle - \langle v_1, v_2 \rangle | \leq \epsilon ||v_1||_2 ||v_2||_2
    \end{equation}
    
\begin{enumerate}
    \item $v_1, v_2 \in \textrm{Column-space}(E)$. We use the same argument as given on page 12 \cite{woodruff} and conclude that condition holds for these $v_1,v_2$
    \item $v_1, v_2 \in \{e_i\}_{i=0}^{n{-}1} \cup 1/2-net$. condition holds due to JLT
    \item $v_1 \in \{e_i\}_{i=0}^{n{-}1} , v_2 \in \textrm{column-space} (E)$. Let $||v_2|| = 1$. We will prove the condition for this case and for all other non-unit-norm cases will happen due to scaling. Using argumnet from page 12 \cite{woodruff}, we can represent $v_2$ as a sum of vectors in $1/2-net$.
    \begin{equation}
        v_2 = v^0 + v^1 + v^2 + .... \; \textrm{ where } v^i \in \textrm{1/2-net}
    \end{equation} 
    such that $||v^i|| \leq \frac{1}{2^i}$
    \begin{align*}
        & |\langle Sv_1, Sv_2 \rangle - \langle v_1, v_2 \rangle | \\
        & = |\sum_i (\langle Sv_1, Sv^i \rangle) - \sum_i (\langle v_1, v^i \rangle) |\\
        & = |\sum_i (\langle Sv_1, Sv^i \rangle) - (\langle v_1, v^i \rangle) |\\
        & \leq \sum_i  (\langle Sv_1, Sv^i \rangle) - (\langle v_1, v^i \rangle) | \\
        &  \leq (\epsilon) \sum_i ||v_1||_2 \; ||v^i||_2 \\
        &  =  \epsilon \sum_i ||v^i||_2 \\
        &  \leq  \epsilon \sum_i \frac{1}{2^i} \\
        &  =  \epsilon  \\
    \end{align*}
\end{enumerate}
Thus, using the matrix S, for all $v1,v2 \in \{e_i\} \cup \textrm{column-space}(E)$ it holds that , 
\begin{equation}
    |\langle Sv_1, Sv_2 \rangle - \langle v_1, v_2 \rangle | \leq \epsilon ||v_1||_2 ||v_2||_2
\end{equation}

\paragraph{ proving the statement for $(1 \pm \epsilon)$ PSS}
Let $e_i$s be one hot encoded vectors such that we have that 
\begin{equation}
    E[i] = e_i^\top E
\end{equation}
Let $M = SE$ and $\mathcal{M}(M,i) = (SE)^\top S e_i$
\begin{align}
    &\langle \mathcal{M}(M,i) , x \rangle - \langle E[i]^\top, x \rangle \\
    & = \langle  (SE)^\top S e_i , x \rangle - \langle E[i]^\top, x \rangle \\
    & = \langle  S e_i , SEx \rangle - \langle (e_i^\top E)^\top, x \rangle
    & = \langle  S e_i , SEx \rangle - \langle e_i, Ex \rangle
\end{align}
Note that both $e_i$ and $Ex$ are vectors that belong to our f point set V. Thus

\begin{align}
    &\langle \mathcal{M}(M,i) , x \rangle - \langle E[i]^\top, x \rangle \\
    & = \langle  S e_i , SEx \rangle - \langle e_i, Ex \rangle \\
    & \leq \epsilon ||e_i||_2 ||Ex||_2 \\
    & \leq \epsilon ||Ex||_2 \\
    & \leq \epsilon \sigma(E) ||x||_2
\end{align}
Note that $||e_i||=1$ Let $\sigma$ be the max singular value of E. We can push $\sigma(E)$ into $\epsilon$ and select a transform $JLT(\epsilon / \sigma(E), \delta, (n + 9^d))$ to obtain a $\epsilon-l_2$ parameter shared setup.

\end{proof}

\section{Proof of theorem \ref{thm:1}}
Let  $(M, \mathcal{M})$ is a $(1 \pm \epsilon)$ PSS for embedding table E, we have ,
\begin{align*}
    \forall i, j \in \{0,...,n{-}1\} \quad 
    | \;\; || \mathcal{M}(M, i) - \mathcal{M}(M, j) ||_2 - || (E[i] - E[j]) ||_2 \; \; | \leq 2\epsilon
\end{align*}

\begin{proof}
Using definition of PSS
\begin{align*}
    | \langle \mathcal{M}(M, i), \mathcal{M}(M, i) \rangle - \langle \mathcal{M}(M, i), E[i] \rangle|   \leq \epsilon ||\mathcal{M}(M, i)||_2
\end{align*}

\begin{align*}
    &\langle \mathcal{M}(M, i), \mathcal{M}(M, i) \rangle  \leq  \langle \mathcal{M}(M, i), E[i] \rangle + \epsilon || \mathcal{M}(M,i)||_2 \\
    &\textrm{ and } \langle \mathcal{M}(M, i), \mathcal{M}(M, i) \rangle  \geq  \langle \mathcal{M}(M, i), E[i] \rangle - \epsilon || \mathcal{M}(M,i)||_2
\end{align*}

\begin{align*}
    &\langle \mathcal{M}(M, i), \mathcal{M}(M, i) \rangle  \leq  \langle E[i], E[i] \rangle + \epsilon (||E[i]||_2  +  || \mathcal{M}(M,i)||_2 )\\
    &\textrm{ and } \langle \mathcal{M}(M, i), \mathcal{M}(M, i) \rangle  \geq  \langle E[i], E[i] \rangle - \epsilon (||E[i]||_2 +  || \mathcal{M}(M,i)||_2)
\end{align*}

\begin{align*}
    &\langle \mathcal{M}(M, i), \mathcal{M}(M, i) \rangle  \leq  \langle E[i], E[i] \rangle + \epsilon (||E[i]||_2  +  || \mathcal{M}(M,i)||_2 )\\
    &\textrm{ and } \langle \mathcal{M}(M, i), \mathcal{M}(M, i) \rangle  \geq  \langle E[i], E[i] \rangle - \epsilon (||E[i]||_2 +  || \mathcal{M}(M,i)||_2)
\end{align*}
Let $||\mathcal{M}(M, i)|| = m$ and $||E[i]|| = e$

\begin{align*}
    m^2 - \epsilon m \leq  e^2 + \epsilon e  \textrm{ and } m^2  + \epsilon m   \geq  e^2  - \epsilon e
\end{align*}
Adding $(1/4)\epsilon^2$ on both sides

\begin{align*}
    m^2 - \epsilon m + 1/4\epsilon^2 \leq  e^2 + \epsilon e + \epsilon^2  \textrm{ and } m^2  + \epsilon m + 1/4 \epsilon^2   \geq  e^2  - \epsilon e + 1/4 \epsilon^2
\end{align*}

\begin{align*}
    (m - 1/2\epsilon)^2 \leq (e + 1/2 \epsilon)^2  \textrm{ and } (m + 1/2 \epsilon)^2   \geq  (e + 1/2 \epsilon)^2
\end{align*}

\begin{align*}
    m \leq (e + \epsilon)  \textrm{ and } m   \geq  e - \epsilon
\end{align*}
Thus,

\begin{align*}
    |m - e| 
    \leq \epsilon
\end{align*}

Thus, 
\begin{equation}
    || \mathcal{M}(M, i) ||_2 =  || E[i] ||_2 \pm \epsilon
\end{equation}

Now we can look at the pairwise distances
\begin{align*}
    & || \mathcal{M}(M, i) - \mathcal{M}(M, j) ||_2 \\
    & = || (\mathcal{M}(M, i) - E[i]) - (\mathcal{M}(M, j) - E[j]) + (E[i] - E[j]) ||_2 \\
    & \leq || (\mathcal{M}(M, i) - E[i]) ||_2 + || (\mathcal{M}(M, j) - E[j])||_2 + || (E[i] - E[j]) ||_2 \\
    & \leq 2\epsilon + || (E[i] - E[j]) ||_2
\end{align*}
    
\begin{align*}
    & || \mathcal{M}(M, i) - \mathcal{M}(M, j) ||_2 \\
    & = || (\mathcal{M}(M, i) - E[i]) - (\mathcal{M}(M, j) - E[j]) + (E[i] - E[j]) ||_2 \\
    & \geq - || (\mathcal{M}(M, i) - E[i]) ||_2 -  || (\mathcal{M}(M, j) - E[j])||_2 + || (E[i] - E[j]) ||_2 \\
    & \geq - 2\epsilon + || (E[i] - E[j]) ||_2
\end{align*}

Thus, 
\begin{align*}
    | \;\; || \mathcal{M}(M, i) - \mathcal{M}(M, j) ||_2 - || (E[i] - E[j]) ||_2 \; \; | \leq 2\epsilon
\end{align*}

\end{proof}

\newpage
\section{Data}
\subsection{PSS rigorous evaluation} \label{sec:appendix_PSS_exp}

\begin{itemize}
    \item Page 1 : memory vs compression.
    \item page 2:
    memory vs chunk size
    \item page 3:
    compression vs chunk size
\end{itemize}
Some obervations 
\begin{itemize}
    \item (sparse/dense) original embeddings should be run with sparse gradients. The dense embeddings are too time consuming as they require updating large amounts of memory (vacuously) in each iteration. Note that a lot of optimizers in deep learning libraries like pytorch do not yet have support for sparse gradient updates. 
    \item When the entire embedding table can fit on the GPU, the forward and backward times do not change much.
    \item While we can go upto 64M tokens ( for m = 128 ) on single gpu using original embeddings and sparse gradient propogation, the simulated embedding tables can be much larger with PSS and compression.
    \item Best forward times we can achieve with PSS ( around 0.27ms/iteration ) is almost 2$\times$ the forward times for original embedding lookup. (0.12 ms / iteration). This is expected since, we have to perform additional hash computations in PSS. Hence, if both original embeddings and compressed embeddings are at the same distance from the computational resource, then PSS will have a disadvantage.
    \item (compression, PSS) Higher the compression, better are timings for both backward pass (using dense propagation) and forward pass.
    \item A nice trade-off can be seen between using sparse or dense gradients with PSS. If we use dense gradients, the time in backward propogation is affected by the overall size of embedding parameters. Thus, for higher compression and smaller $n$, we have smaller times where as if n becomes larger and compression is smaller, the time taken increases. The cost of dense gradient propogation becomes quite high at sufficiently large $n$ and sufficiently small compression. On the other hand, the cost of sparse gradient propogation is uniform across different values of $n$ and compression. This is because the algorithmic complexity and memory accessed is similar. This gives us a guideline as to when to use sparse / dense gradients with PSS. Note that in our final results, we see a good overall improvement by leveraging the dense gradient propogation at high compression.
    \item Higher the chunk size better is the time in forward pass. backward pass is largely unaffected by chunk size largely due to the way it is implemented.
\end{itemize}
\begin{table}[h]
\centering
\begin{tabular}{|cccccccl|}
\hline
\multicolumn{8}{|c|}{\textbf{Original time(ms)}}                                                                                                                                                                                                                                                                                                                                                                                                                                                                                                                                       \\ \hline
\rowcolor[HTML]{FFFFFF} 
\multicolumn{1}{|c|}{\cellcolor[HTML]{FFFFFF}{\color[HTML]{333333} \textit{}}}     & \multicolumn{1}{c|}{\cellcolor[HTML]{FFFFFF}{\color[HTML]{333333} }}                & \multicolumn{6}{c|}{\cellcolor[HTML]{FFFFFF}{\color[HTML]{333333} \textbf{n}}}                                                                                                                                                                                                                                                                                                                              \\ \hline
\rowcolor[HTML]{FFFFFF} 
\multicolumn{1}{|c|}{\cellcolor[HTML]{FFFFFF}{\color[HTML]{333333} \textbf{call}}} & \multicolumn{1}{c|}{\cellcolor[HTML]{FFFFFF}{\color[HTML]{333333} \textbf{sparse}}} & \multicolumn{1}{c|}{\cellcolor[HTML]{FFFFFF}{\color[HTML]{333333} 4M}} & \multicolumn{1}{c|}{\cellcolor[HTML]{FFFFFF}{\color[HTML]{333333} 8M}} & \multicolumn{1}{c|}{\cellcolor[HTML]{FFFFFF}{\color[HTML]{333333} 16M}} & \multicolumn{1}{c|}{\cellcolor[HTML]{FFFFFF}{\color[HTML]{333333} 32M}} & \multicolumn{1}{c|}{\cellcolor[HTML]{FFFFFF}{\color[HTML]{333333} 64M}} & {\color[HTML]{333333} 128M} \\
\hline
\end{tabular}
\caption{Time taken for original embeddingn table}
\label{tab:t0}
\end{table}
\newpage

\begin{table}[h]
\centering
\begin{tabular}{|crrrr|}
\hline
\multicolumn{5}{|c|}{\textbf{Forward pass (ms) {[}sparse=false{]}}}                                                                                                                                                                                                                                                                                                    \\ \hline
\rowcolor[HTML]{FFFFFF} 
\multicolumn{1}{|c|}{\cellcolor[HTML]{FFFFFF}\textit{}}  & \multicolumn{4}{c|}{\cellcolor[HTML]{FFFFFF}\textbf{compression}}                                                                                                                                                                                                                                           \\ \hline
\rowcolor[HTML]{FFFFFF} 
\multicolumn{1}{|c|}{\cellcolor[HTML]{FFFFFF}\textbf{n}} & \multicolumn{1}{c|}{\cellcolor[HTML]{FFFFFF}{\color[HTML]{000000} 10x}} & \multicolumn{1}{c|}{\cellcolor[HTML]{FFFFFF}{\color[HTML]{000000} 100x}} & \multicolumn{1}{c|}{\cellcolor[HTML]{FFFFFF}{\color[HTML]{000000} 1000x}} & \multicolumn{1}{c|}{\cellcolor[HTML]{FFFFFF}{\color[HTML]{000000} 10000x}} \\ \hline
\multicolumn{1}{|c|}{\cellcolor[HTML]{FFFFFF}4M}         & \multicolumn{1}{r|}{\cellcolor[HTML]{FBC768}0.40}                       & \multicolumn{1}{r|}{\cellcolor[HTML]{B0C977}0.31}                        & \multicolumn{1}{r|}{\cellcolor[HTML]{7EC182}0.27}                         & \cellcolor[HTML]{70BF85}0.27                                               \\ \hline
\multicolumn{1}{|c|}{\cellcolor[HTML]{FFFFFF}8M}         & \multicolumn{1}{r|}{\cellcolor[HTML]{FFD666}0.36}                       & \multicolumn{1}{r|}{\cellcolor[HTML]{BACA75}0.32}                        & \multicolumn{1}{r|}{\cellcolor[HTML]{92C47E}0.29}                         & \cellcolor[HTML]{76C084}0.27                                               \\
\hline
\end{tabular}
\caption{ batch size = 10240 and dimension = 128 }
\label{tab:t1}
\end{table}

\begin{table}[h]
\centering
\begin{tabular}{|crrrr|}
\hline
\multicolumn{5}{|c|}{\textbf{Forward pass (ms) {[}sparse=true{]}}}                                                                                                                                                                                                                                                                                                     \\ \hline
\rowcolor[HTML]{FFFFFF} 
\multicolumn{1}{|c|}{\cellcolor[HTML]{FFFFFF}\textit{}}  & \multicolumn{4}{c|}{\cellcolor[HTML]{FFFFFF}\textbf{compression}}                                                                                                                                                                                                                                           \\ \hline
\rowcolor[HTML]{FFFFFF} 
\multicolumn{1}{|c|}{\cellcolor[HTML]{FFFFFF}\textbf{n}} & \multicolumn{1}{c|}{\cellcolor[HTML]{FFFFFF}{\color[HTML]{000000} 10x}} & \multicolumn{1}{c|}{\cellcolor[HTML]{FFFFFF}{\color[HTML]{000000} 100x}} & \multicolumn{1}{c|}{\cellcolor[HTML]{FFFFFF}{\color[HTML]{000000} 1000x}} & \multicolumn{1}{c|}{\cellcolor[HTML]{FFFFFF}{\color[HTML]{000000} 10000x}} \\ \hline
\multicolumn{1}{|c|}{\cellcolor[HTML]{FFFFFF}4M}         & \multicolumn{1}{r|}{\cellcolor[HTML]{E77D72}0.57}                       & \multicolumn{1}{r|}{\cellcolor[HTML]{EB8B70}0.54}                        & \multicolumn{1}{r|}{\cellcolor[HTML]{DCD06E}0.34}                         & \cellcolor[HTML]{D0CE70}0.33                                               \\ \hline
\multicolumn{1}{|c|}{\cellcolor[HTML]{FFFFFF}8M}         & \multicolumn{1}{r|}{\cellcolor[HTML]{ED956F}0.51}                       & \multicolumn{1}{r|}{\cellcolor[HTML]{E77E72}0.56}                        & \multicolumn{1}{r|}{\cellcolor[HTML]{E3D16C}0.34}                         & \cellcolor[HTML]{C5CC73}0.32                                               \\
\hline
\end{tabular}
\caption{ batch size = 10240 and dimension = 128 }
\label{tab:t2}
\end{table}



\begin{table}[h]
\centering
\begin{tabular}{|crrrr|}
\hline
\multicolumn{5}{|c|}{\textbf{Backward pass (ms) {[}sparse=false{]}}}                                                                                                                                                                                                                                                                                                    \\ \hline
\rowcolor[HTML]{FFFFFF} 
\multicolumn{1}{|c|}{\cellcolor[HTML]{FFFFFF}\textit{}}  & \multicolumn{4}{c|}{\cellcolor[HTML]{FFFFFF}\textbf{compression}}                                                                                                                                                                                                                                           \\ \hline
\rowcolor[HTML]{FFFFFF} 
\multicolumn{1}{|c|}{\cellcolor[HTML]{FFFFFF}\textbf{n}} & \multicolumn{1}{c|}{\cellcolor[HTML]{FFFFFF}{\color[HTML]{000000} 10x}} & \multicolumn{1}{c|}{\cellcolor[HTML]{FFFFFF}{\color[HTML]{000000} 100x}} & \multicolumn{1}{c|}{\cellcolor[HTML]{FFFFFF}{\color[HTML]{000000} 1000x}} & \multicolumn{1}{c|}{\cellcolor[HTML]{FFFFFF}{\color[HTML]{000000} 10000x}} \\ \hline
\multicolumn{1}{|c|}{\cellcolor[HTML]{FFFFFF}4M}         & \multicolumn{1}{r|}{\cellcolor[HTML]{FFD666}2.55}                       & \multicolumn{1}{r|}{\cellcolor[HTML]{6ABE86}0.59}                        & \multicolumn{1}{r|}{\cellcolor[HTML]{57BB8A}0.35}                         & \cellcolor[HTML]{57BB8A}0.35                                               \\ \hline
\multicolumn{1}{|c|}{\cellcolor[HTML]{FFFFFF}8M}         & \multicolumn{1}{r|}{\cellcolor[HTML]{FFD366}4.81}                       & \multicolumn{1}{r|}{\cellcolor[HTML]{7CC182}0.83}                        & \multicolumn{1}{r|}{\cellcolor[HTML]{58BB8A}0.37}                         & \cellcolor[HTML]{57BB8A}0.35                                               \\
\hline
\end{tabular}
\caption{ batch size = 10240 and dimension = 128 }
\label{tab:t3}
\end{table}


\begin{table}[h]
\centering
\begin{tabular}{|crrrr|}
\hline
\multicolumn{5}{|c|}{\textbf{Backward pass (ms) {[}sparse=true{]}}}                                                                                                                                                                                                                                                                                                   \\ \hline
\rowcolor[HTML]{FFFFFF} 
\multicolumn{1}{|c|}{\cellcolor[HTML]{FFFFFF}\textit{}}  & \multicolumn{4}{c|}{\cellcolor[HTML]{FFFFFF}\textbf{compression}}                                                                                                                                                                                                                                           \\ \hline
\rowcolor[HTML]{FFFFFF} 
\multicolumn{1}{|c|}{\cellcolor[HTML]{FFFFFF}\textbf{n}} & \multicolumn{1}{c|}{\cellcolor[HTML]{FFFFFF}{\color[HTML]{000000} 10x}} & \multicolumn{1}{c|}{\cellcolor[HTML]{FFFFFF}{\color[HTML]{000000} 100x}} & \multicolumn{1}{c|}{\cellcolor[HTML]{FFFFFF}{\color[HTML]{000000} 1000x}} & \multicolumn{1}{c|}{\cellcolor[HTML]{FFFFFF}{\color[HTML]{000000} 10000x}} \\ \hline
\multicolumn{1}{|c|}{\cellcolor[HTML]{FFFFFF}4M}         & \multicolumn{1}{r|}{\cellcolor[HTML]{FFD666}2.67}                       & \multicolumn{1}{r|}{\cellcolor[HTML]{FBD567}2.48}                        & \multicolumn{1}{r|}{\cellcolor[HTML]{FFD666}2.56}                         & \cellcolor[HTML]{FED567}2.52                                               \\ \hline
\rowcolor[HTML]{FFD666} 
\multicolumn{1}{|c|}{\cellcolor[HTML]{FFFFFF}8M}         & \multicolumn{1}{r|}{\cellcolor[HTML]{FFD666}2.53}                       & \multicolumn{1}{r|}{\cellcolor[HTML]{FFD666}2.58}                        & \multicolumn{1}{r|}{\cellcolor[HTML]{FBD567}2.48}                         & 2.53                                                                       \\
\hline
\end{tabular}
\caption{ batch size = 10240 and dimension = 128 }
\label{tab:t4}
\end{table}

\newpage

\begin{table}[h]
\centering
\begin{tabular}{|crrrrrr|}
\hline
\multicolumn{7}{|c|}{\textbf{Forward pass (ms) {[}sparse=false{]}}}                                                                                                                                                                                                                                                                                                                                                                         \\ \hline
\rowcolor[HTML]{FFFFFF} 
\multicolumn{1}{|c|}{\cellcolor[HTML]{FFFFFF}\textit{}}  & \multicolumn{6}{c|}{\cellcolor[HTML]{FFFFFF}\textbf{PSS chunk size}}                                                                                                                                                                                                                                                                                                          \\ \hline
\multicolumn{1}{|c|}{\cellcolor[HTML]{FFFFFF}\textbf{n}} & \multicolumn{1}{c|}{\cellcolor[HTML]{FFFFFF}{\color[HTML]{000000} 1}} & \multicolumn{1}{c|}{\cellcolor[HTML]{FFFFFF}{\color[HTML]{000000} 2}} & \multicolumn{1}{c|}{\cellcolor[HTML]{FFFFFF}{\color[HTML]{000000} 4}} & \multicolumn{1}{c|}{\cellcolor[HTML]{FFFFFF}{\color[HTML]{000000} 8}} & \multicolumn{1}{l|}{16}                           & \multicolumn{1}{l|}{32}      \\ \hline
\multicolumn{1}{|c|}{\cellcolor[HTML]{FFFFFF}4M}         & \multicolumn{1}{r|}{\cellcolor[HTML]{FAC169}0.68}                     & \multicolumn{1}{r|}{\cellcolor[HTML]{F7D468}0.58}                     & \multicolumn{1}{r|}{\cellcolor[HTML]{E5D16C}0.55}                     & \multicolumn{1}{r|}{\cellcolor[HTML]{7CC182}0.32}                     & \multicolumn{1}{r|}{\cellcolor[HTML]{79C083}0.31} & \cellcolor[HTML]{7AC083}0.31 \\
\hline
\end{tabular}
\caption{ batch size = 10240 and dimension = 128 }
\label{tab:t5}
\end{table}

\begin{table}[h]
\centering
\begin{tabular}{|crrrrrr|}
\hline
\multicolumn{7}{|c|}{\textbf{Forward pass (ms) {[}sparse=true{]}}}                                                                                                                                                                                                                                                                                                                                                                                                        \\ \hline
\rowcolor[HTML]{FFFFFF} 
\multicolumn{1}{|c|}{\cellcolor[HTML]{FFFFFF}\textit{}}           & \multicolumn{6}{c|}{\cellcolor[HTML]{FFFFFF}\textbf{PSS chunk size}}                                                                                                                                                                                                                                                                                                                               \\ \hline
\rowcolor[HTML]{8093B3} 
\multicolumn{1}{|l|}{\cellcolor[HTML]{DFE4EC}\textit{\textbf{n}}} & \multicolumn{1}{r|}{\cellcolor[HTML]{8093B3}{\color[HTML]{FFFFFF} 1}} & \multicolumn{1}{r|}{\cellcolor[HTML]{8093B3}{\color[HTML]{FFFFFF} 2}} & \multicolumn{1}{r|}{\cellcolor[HTML]{8093B3}{\color[HTML]{FFFFFF} 4}} & \multicolumn{1}{r|}{\cellcolor[HTML]{8093B3}{\color[HTML]{FFFFFF} 8}} & \multicolumn{1}{r|}{\cellcolor[HTML]{8093B3}{\color[HTML]{FFFFFF} 16}} & {\color[HTML]{FFFFFF} 32}    \\ \hline
\multicolumn{1}{|c|}{\cellcolor[HTML]{FFFFFF}4M}                  & \multicolumn{1}{r|}{\cellcolor[HTML]{EB8E70}0.87}                     & \multicolumn{1}{r|}{\cellcolor[HTML]{FED166}0.62}                     & \multicolumn{1}{r|}{\cellcolor[HTML]{EFD36A}0.57}                     & \multicolumn{1}{r|}{\cellcolor[HTML]{EFD36A}0.57}                     & \multicolumn{1}{r|}{\cellcolor[HTML]{D8CF6F}0.52}                      & \cellcolor[HTML]{95C57D}0.37 \\
\hline
\end{tabular}
\caption{ batch size = 10240 and dimension = 128, fixed compression = 100x }
\label{tab:t6}
\end{table}

\begin{table}[h]
\centering
\begin{tabular}{|crrrrrr|}
\hline
\multicolumn{7}{|c|}{\textbf{Backward pass (ms) {[}sparse=false{]}}}                                                                                                                                                                                                                                                                                                                                                                                                       \\ \hline
\rowcolor[HTML]{FFFFFF} 
\multicolumn{1}{|c|}{\cellcolor[HTML]{FFFFFF}\textit{}}           & \multicolumn{6}{c|}{\cellcolor[HTML]{FFFFFF}\textbf{PSS chunk size}}                                                                                                                                                                                                                                                                                                                                \\ \hline
\rowcolor[HTML]{8093B3} 
\multicolumn{1}{|l|}{\cellcolor[HTML]{DFE4EC}\textit{\textbf{n}}} & \multicolumn{1}{r|}{\cellcolor[HTML]{8093B3}{\color[HTML]{FFFFFF} 1}} & \multicolumn{1}{r|}{\cellcolor[HTML]{8093B3}{\color[HTML]{FFFFFF} 2}} & \multicolumn{1}{r|}{\cellcolor[HTML]{8093B3}{\color[HTML]{FFFFFF} 4}} & \multicolumn{1}{r|}{\cellcolor[HTML]{8093B3}{\color[HTML]{FFFFFF} 8}} & \multicolumn{1}{r|}{\cellcolor[HTML]{8093B3}{\color[HTML]{FFFFFF} 16}} & {\color[HTML]{FFFFFF} 32}     \\ \hline
\multicolumn{1}{|c|}{\cellcolor[HTML]{FFFFFF}4M}                  & \multicolumn{1}{r|}{\cellcolor[HTML]{91C47E}1.29}                     & \multicolumn{1}{r|}{\cellcolor[HTML]{81C181}1.10}                     & \multicolumn{1}{r|}{\cellcolor[HTML]{69BE86}0.80}                     & \multicolumn{1}{r|}{\cellcolor[HTML]{5CBB89}0.64}                     & \multicolumn{1}{r|}{\cellcolor[HTML]{57BB8A}0.58}                      & \cellcolor[HTML]{57BB8A}0.57  \\
\hline
\end{tabular}
\caption{ batch size = 10240 and dimension = 128, fixed compression = 100x }
\label{tab:t7}
\end{table}

\begin{table}[h]
\centering
\begin{tabular}{|crrrrrr|}
\hline
\multicolumn{7}{|c|}{\textbf{Backward pass (ms) {[}sparse=true{]}}}                                                                                                                                                                                                                                                                                                                                                                                                       \\ \hline
\rowcolor[HTML]{FFFFFF} 
\multicolumn{1}{|c|}{\cellcolor[HTML]{FFFFFF}\textit{}}           & \multicolumn{6}{c|}{\cellcolor[HTML]{FFFFFF}\textbf{PSS chunk size}}                                                                                                                                                                                                                                                                                                                               \\ \hline
\rowcolor[HTML]{8093B3} 
\multicolumn{1}{|l|}{\cellcolor[HTML]{DFE4EC}\textit{\textbf{n}}} & \multicolumn{1}{r|}{\cellcolor[HTML]{8093B3}{\color[HTML]{FFFFFF} 1}} & \multicolumn{1}{r|}{\cellcolor[HTML]{8093B3}{\color[HTML]{FFFFFF} 2}} & \multicolumn{1}{r|}{\cellcolor[HTML]{8093B3}{\color[HTML]{FFFFFF} 4}} & \multicolumn{1}{r|}{\cellcolor[HTML]{8093B3}{\color[HTML]{FFFFFF} 8}} & \multicolumn{1}{r|}{\cellcolor[HTML]{8093B3}{\color[HTML]{FFFFFF} 16}} & {\color[HTML]{FFFFFF} 32}    \\ \hline
\multicolumn{1}{|c|}{\cellcolor[HTML]{FFFFFF}4M}                  & \multicolumn{1}{r|}{\cellcolor[HTML]{FED366}3.17}                     & \multicolumn{1}{r|}{\cellcolor[HTML]{FFD466}3.04}                     & \multicolumn{1}{r|}{\cellcolor[HTML]{FCD567}2.61}                     & \multicolumn{1}{r|}{\cellcolor[HTML]{F7D468}2.54}                     & \multicolumn{1}{r|}{\cellcolor[HTML]{F3D469}2.50}                      & \cellcolor[HTML]{F1D369}2.48 \\
\hline
\end{tabular}
\caption{ batch size = 10240 and dimension = 128, fixed compression = 100x }
\label{tab:t8}
\end{table}

\begin{table}[h]
\centering
\begin{tabular}{|ccccc|}
\hline
\multicolumn{5}{|c|}{\textbf{Forward pass (ms) {[}sparse=false{]}}}                                                                                                                                                                                                                                                           \\ \hline
\rowcolor[HTML]{FFFFFF} 
\multicolumn{1}{|c|}{\cellcolor[HTML]{FFFFFF}\textit{}}      & \multicolumn{4}{c|}{\cellcolor[HTML]{FFFFFF}\textbf{compression}}                                                                                                                                                                                              \\ \hline
\rowcolor[HTML]{FFFFFF} 
\multicolumn{1}{|c|}{\cellcolor[HTML]{FFFFFF}\textbf{chunk}} & \multicolumn{1}{c|}{\cellcolor[HTML]{FFFFFF}{\color[HTML]{000000} 10x}} & \multicolumn{1}{c|}{\cellcolor[HTML]{FFFFFF}{\color[HTML]{000000} 100x}} & \multicolumn{1}{c|}{\cellcolor[HTML]{FFFFFF}{\color[HTML]{000000} 1000x}} & {\color[HTML]{000000} 10000x} \\ \hline
\multicolumn{1}{|c|}{\cellcolor[HTML]{F4F6F8}1}              & \multicolumn{1}{c|}{\cellcolor[HTML]{E98671}0.89}                       & \multicolumn{1}{c|}{\cellcolor[HTML]{E88272}0.90}                        & \multicolumn{1}{c|}{\cellcolor[HTML]{F3AA6C}0.73}                         & \cellcolor[HTML]{5DBC89}0.28  \\ \hline
\multicolumn{1}{|c|}{\cellcolor[HTML]{F4F6F8}2}              & \multicolumn{1}{c|}{\cellcolor[HTML]{F1A36D}0.77}                       & \multicolumn{1}{c|}{\cellcolor[HTML]{EF9C6E}0.80}                        & \multicolumn{1}{c|}{\cellcolor[HTML]{FBC768}0.61}                         & \cellcolor[HTML]{57BB8A}0.27  \\
\hline
\end{tabular}
\caption{ batch size = 10240 and dimension = 128 }
\label{tab:t9}
\end{table}

\begin{table}[h]
\centering
\begin{tabular}{|crrrr|}
\hline
\multicolumn{5}{|c|}{\textbf{Forward pass (ms) {[}sparse=true{]}}}                                                                                                                                                                                                                                                                                                         \\ \hline
\rowcolor[HTML]{FFFFFF} 
\multicolumn{1}{|c|}{\cellcolor[HTML]{FFFFFF}\textit{}}      & \multicolumn{4}{c|}{\cellcolor[HTML]{FFFFFF}\textbf{compression}}                                                                                                                                                                                                                                           \\ \hline
\rowcolor[HTML]{FFFFFF} 
\multicolumn{1}{|c|}{\cellcolor[HTML]{FFFFFF}\textbf{chunk}} & \multicolumn{1}{c|}{\cellcolor[HTML]{FFFFFF}{\color[HTML]{000000} 10x}} & \multicolumn{1}{c|}{\cellcolor[HTML]{FFFFFF}{\color[HTML]{000000} 100x}} & \multicolumn{1}{c|}{\cellcolor[HTML]{FFFFFF}{\color[HTML]{000000} 1000x}} & \multicolumn{1}{c|}{\cellcolor[HTML]{FFFFFF}{\color[HTML]{000000} 10000x}} \\ \hline
\multicolumn{1}{|c|}{\cellcolor[HTML]{F4F6F8}1}              & \multicolumn{1}{r|}{\cellcolor[HTML]{E77D72}0.93}                       & \multicolumn{1}{r|}{\cellcolor[HTML]{E67C73}0.93}                        & \multicolumn{1}{r|}{\cellcolor[HTML]{EA8B70}0.87}                         & \cellcolor[HTML]{80C182}0.34                                               \\ \hline
\multicolumn{1}{|c|}{\cellcolor[HTML]{F4F6F8}2}              & \multicolumn{1}{r|}{\cellcolor[HTML]{EF9A6E}0.80}                       & \multicolumn{1}{r|}{\cellcolor[HTML]{EF9C6E}0.80}                        & \multicolumn{1}{r|}{\cellcolor[HTML]{F1A16D}0.77}                         & \cellcolor[HTML]{75BF84}0.32                                               \\
\hline
\end{tabular}
\caption{ batch size = 10240 and dimension = 128 }
\label{tab:t10}
\end{table}

\begin{table}[h]
\centering
\begin{tabular}{|crrrr|}
\hline
\multicolumn{5}{|c|}{\textbf{Backward pass (ms) {[}sparse=false{]}}}                                                                                                                                                                                                                                                                                                        \\ \hline
\rowcolor[HTML]{FFFFFF} 
\multicolumn{1}{|c|}{\cellcolor[HTML]{FFFFFF}\textit{}}      & \multicolumn{4}{c|}{\cellcolor[HTML]{FFFFFF}\textbf{compression}}                                                                                                                                                                                                                                           \\ \hline
\rowcolor[HTML]{FFFFFF} 
\multicolumn{1}{|c|}{\cellcolor[HTML]{FFFFFF}\textbf{chunk}} & \multicolumn{1}{c|}{\cellcolor[HTML]{FFFFFF}{\color[HTML]{000000} 10x}} & \multicolumn{1}{c|}{\cellcolor[HTML]{FFFFFF}{\color[HTML]{000000} 100x}} & \multicolumn{1}{c|}{\cellcolor[HTML]{FFFFFF}{\color[HTML]{000000} 1000x}} & \multicolumn{1}{c|}{\cellcolor[HTML]{FFFFFF}{\color[HTML]{000000} 10000x}} \\ \hline
\multicolumn{1}{|c|}{\cellcolor[HTML]{F4F6F8}1}              & \multicolumn{1}{r|}{\cellcolor[HTML]{E67C73}35.94}                      & \multicolumn{1}{r|}{\cellcolor[HTML]{FED166}4.69}                        & \multicolumn{1}{r|}{\cellcolor[HTML]{A7C779}1.44}                         & \cellcolor[HTML]{58BB8A}0.37                                               \\ \hline
\multicolumn{1}{|c|}{\cellcolor[HTML]{F4F6F8}2}              & \multicolumn{1}{r|}{\cellcolor[HTML]{E77D72}35.78}                      & \multicolumn{1}{r|}{\cellcolor[HTML]{FED166}4.53}                        & \multicolumn{1}{r|}{\cellcolor[HTML]{99C57C}1.26}                         & \cellcolor[HTML]{57BB8A}0.37                                               \\
\hline
\end{tabular}
\caption{ batch size = 10240 and dimension = 128 }
\label{tab:t11}
\end{table}

\begin{table}[h]
\centering
\begin{tabular}{|crrrr|}
\hline
\multicolumn{5}{|c|}{\textbf{Backward pass (ms) {[}sparse=true{]}}}                                                                                                                                                                                                                                                                                                        \\ \hline
\rowcolor[HTML]{FFFFFF} 
\multicolumn{1}{|c|}{\cellcolor[HTML]{FFFFFF}\textit{}}      & \multicolumn{4}{c|}{\cellcolor[HTML]{FFFFFF}\textbf{compression}}                                                                                                                                                                                                                                           \\ \hline
\rowcolor[HTML]{FFFFFF} 
\multicolumn{1}{|c|}{\cellcolor[HTML]{FFFFFF}\textbf{chunk}} & \multicolumn{1}{c|}{\cellcolor[HTML]{FFFFFF}{\color[HTML]{000000} 10x}} & \multicolumn{1}{c|}{\cellcolor[HTML]{FFFFFF}{\color[HTML]{000000} 100x}} & \multicolumn{1}{c|}{\cellcolor[HTML]{FFFFFF}{\color[HTML]{000000} 1000x}} & \multicolumn{1}{c|}{\cellcolor[HTML]{FFFFFF}{\color[HTML]{000000} 10000x}} \\ \hline
\multicolumn{1}{|c|}{\cellcolor[HTML]{F4F6F8}1}              & \multicolumn{1}{r|}{\cellcolor[HTML]{FED366}4.08}                       & \multicolumn{1}{r|}{\cellcolor[HTML]{FFD366}3.89}                        & \multicolumn{1}{r|}{\cellcolor[HTML]{FFD666}2.77}                         & \cellcolor[HTML]{FFD666}2.87                                               \\ \hline
\multicolumn{1}{|c|}{\cellcolor[HTML]{F4F6F8}2}              & \multicolumn{1}{r|}{\cellcolor[HTML]{FFD566}3.32}                       & \multicolumn{1}{r|}{\cellcolor[HTML]{FFD566}3.20}                        & \multicolumn{1}{r|}{\cellcolor[HTML]{FFD666}2.76}                         & \cellcolor[HTML]{FFD666}2.88                                               \\
\hline
\end{tabular}
\caption{ batch size = 10240 and dimension = 128 }
\label{tab:t12}
\end{table}

\newpage

\end{document}